\documentclass[10pt,journal,compsoc]{IEEEtran}
\pdfoutput=1

\ifCLASSOPTIONcompsoc
    \usepackage[nocompress]{cite}
\else
    \usepackage{cite}
\fi

\ifCLASSINFOpdf
    \usepackage[pdftex]{graphicx}
\else
    \usepackage[dvips]{graphicx}
\fi

\usepackage{amsmath}

\usepackage{array}

\usepackage{url}

\usepackage[utf8]{inputenc}
\usepackage{amsfonts}
\usepackage{amsthm}
\usepackage{array}
\usepackage{mathtools}
\usepackage{gensymb}
\usepackage{textcomp}
\usepackage[english]{babel}

\makeatletter
\let\MYcaption\@makecaption
\makeatother
\usepackage[font=footnotesize]{subcaption}
\makeatletter
\let\@makecaption\MYcaption
\makeatother

\usepackage{xfrac}
\usepackage{csquotes}
\usepackage[super]{nth}

\usepackage{booktabs}
\usepackage{color}

\usepackage{hyperref}
\usepackage[capitalise,noabbrev]{cleveref}

\newcommand{\divergence}[1]{\mathfrak{D}_\text{#1}}
\newcommand{\edivergence}[1]{\widetilde{\divergence{#1}}}

\DeclareMathOperator{\trace}{trace}
\DeclareMathOperator*{\argmax}{argmax}

\crefname{appsec}{Appendix}{Appendices}

\begin{document}

    \title{Detecting Regions of Maximal Divergence for Spatio-Temporal Anomaly Detection}
    
    \author{Bj{\"o}rn~Barz,
        Erik~Rodner,
        Yanira~Guanche~Garcia,
        and~Joachim~Denzler,~\IEEEmembership{Member,~IEEE}% <-this % stops a space
        \IEEEcompsocitemizethanks{\IEEEcompsocthanksitem B.\ Barz, Y.\ Guanche Garcia, and J.\ Denzler are with the Computer Vision Group, Department of Mathematics and Computer Science, Friedrich Schiller University Jena, 07737 Jena, Germany.\protect\\
            E-mail: \{bjoern.barz,yanira.guanche.garcia,joachim.denzler\}@uni-jena.de
            \IEEEcompsocthanksitem E.\ Rodner is with Carl Zeiss AG, Corporate Research and Technology.}% <-this % stops an unwanted space
        }
    
    %\markboth{IEEE Transactions on Pattern Analysis and Machine Intelligence,~Preprint}%~Vol.~X, No.~Y, Month~Year}%
    %{Barz \MakeLowercase{\textit{et al.}}: Detecting Regions of Maximal Divergence}
    
    % Publisher's ID mark
    \IEEEpubid{\textcopyright\ 2018 IEEE.
    Published in IEEE Transactions on Pattern Analysis and Machine Intelligence.
    DOI: 10.1109/TPAMI.2018.2823766.}

    \IEEEtitleabstractindextext{%
        \begin{abstract}
            Automatic detection of anomalies in space- and time-varying measurements is an important tool in several fields, e.g., fraud detection, climate analysis, or healthcare monitoring.
            We present an algorithm for detecting anomalous regions in multivariate spatio-temporal time-series, which allows for spotting the interesting parts in large amounts of data, including video and text data.
            In opposition to existing techniques for detecting isolated anomalous data points, we propose the ``Maximally Divergent Intervals'' (MDI) framework for unsupervised detection of coherent spatial regions and time intervals characterized by a high Kullback-Leibler divergence compared with all other data given.
            In this regard, we define an unbiased Kullback-Leibler divergence that allows for ranking regions of different size and show how to enable the algorithm to run on large-scale data sets in reasonable time using an interval proposal technique.
            Experiments on both synthetic and real data from various domains, such as climate analysis, video surveillance, and text forensics, demonstrate that our method is widely applicable and a valuable tool for finding interesting events in different types of data.
        \end{abstract}
        
        \begin{IEEEkeywords}
            anomaly detection, time series analysis, spatio-temporal data, data mining, unsupervised machine learning
        \end{IEEEkeywords}}

        \maketitle

        \IEEEdisplaynontitleabstractindextext

        % \fi
        \IEEEpeerreviewmaketitle

\IEEEraisesectionheading{\section{Introduction}\label{sec:intro}}

\IEEEPARstart{M}{any} pattern recognition methods strive towards deriving models from complex and noisy data. Such models try to describe the prototypical normal behavior of the system being observed, which is hard to model manually and whose state is often not even directly observable, but only reflected by the data. They allow reasoning about the properties of the system, predicting unseen data, and assessing the ``normality'' of new data. In such a scenario, any deviation from the normal behavior present in the data is distracting and may impair the accuracy of the model. An entire arsenal of techniques has therefore been developed to eliminate abnormal observations prior to learning or to learn models in a robust way not affected by a few anomalies.

Such practices may easily lead to the perception of anomalies as being intrinsically bad and worthless. Though that is true for random noise and erroneous measurements, there may also be anomalies caused by rare events and complex processes. Embracing the anomalies in the data and studying the information buried in them can therefore lead to a deeper understanding of the system being analyzed and to the insight that the models hitherto employed were incomplete or---in the case of non-stationary processes---outdated. A well-known example for this is the discovery of the correlation between the \textit{El Niño} weather phenomenon and extreme surface pressures over the equator by Gilbert Walker \cite{walker1928world} during the early \nth{20} century through the analysis of extreme events in time-series of climate data.

Thus, the use of anomaly detection techniques is not limited to outlier removal as a pre-processing step. In contrast, anomaly detection also is an important task \textit{per se}, since only the deviations from normal behavior are the actual object of interest in many applications. Besides the scenario of knowledge discovery mentioned above, fraud detection (e.g., credit card fraud or identity theft), intrusion detection in cyber-security, fault detection in industrial processes, anomaly detection in healthcare (e.g., monitoring patient condition or detecting disease outbreaks), and early detection of environmental disasters are other important examples. Automated methods for anomaly detection are especially crucial nowadays, where huge amounts of data are available that cannot be analyzed by humans.

In this article, we introduce a novel unsupervised method called ``Maximally Divergent Intervals'' (MDI), which can be employed to point the expert analysts to the interesting parts of the data, i.e., the anomalies.
In contrast to most existing anomaly detection techniques (e.g., \cite{breunig2000lof,kim2012rkde,macgregor1995statistical,schoelkopf2001ocsvm}), we do not analyze the data on a point-wise basis, but search for contiguous intervals of time and regions in space that contain the anomalous event.
This is motivated by the fact that anomalies driven by natural processes rather occur over a space of time and, in the case of spatio-temporal data, in a spatial region rather than at a single location at a single time.
Moreover, the individual samples making up such a so-called \textit{collective anomaly} do not have to be anomalous when considered in isolation, but may be an anomaly only as a whole.
Thus, analysts will intuitively be searching for anomalous {\em regions} in the data instead of anomalous points and the algorithm assisting them should do so as well.

We achieve this by searching for anomalous {\em blocks} in multivariate spatio-temporal data tensors, i.e., regions and time frames whose data distribution deviates most from the distribution of the remaining time-series.
To this end, we compare several existing measures for the divergence of distributions and derive a new one that is invariant against varying length of the intervals being compared.
A fast novel interval proposal technique allows us to reduce the computational cost of this procedure by just analyzing a small portion of particularly interesting parts of the data.
Experiments on climate data, videos, and text corpora will demonstrate that our method can be applied to a variety of applications without major adaptations.

Despite the importance of this task across domains, there has been very limited research on the detection of anomalous intervals in multivariate time-series data, though this problem has been known for a couple of years: Keogh et al.\ \cite{keogh2005hot} have already tackled this task in 2005 with a method they called ``HOT SAX''. They try to find anomalous sub-sequences (``discords'') of time-series by representing all possible sub-sequences of length $d$ as a $d$-dimensional vector and using the Euclidean distance to the nearest neighbor in that space as anomaly score.
More recently, Ren et al.\ \cite{ren2018anomaly} use hand-crafted interval features based on the frequency of extreme values and search for intervals whose features are maximally different from all other intervals.
However, both methods are limited to univariate data and a fixed length of the intervals must be specified in advance.

The latter is also true for a multivariate approach proposed by Liu et al.\ \cite{liu2013change} who compare two consecutive intervals of fixed size in a time-series using the Kullback-Leibler or the Pearson divergence for detecting {\em change-point anomalies}, i.e., points where a permanent change of the distribution of the data occurs. This is a different task than finding intervals that are anomalous with regard to {\em all} the remaining data. In addition, their method does not scale well for detecting anomalous intervals of {\em dynamic} size and is hence not applicable for detecting other types of anomalies, for which a broader context has to be taken into account.

The task of detecting anomalous intervals of dynamic size has recently been tackled by Senin et al.\ \cite{senin2018grammarviz}, who search for typical and anomalous patterns in time-series by inducing a grammar on a symbolic discretization of the data.
As opposed to our approach, their method cannot handle multivariate or spatio-temporal data.

Similar to our approach, Jiang et al.\ \cite{jiang2015general} search for anomalous blocks in higher-order tensors using the Kullback-Leibler divergence, but apply their method to discrete data only (e.g., relations in social networks) and use a Poisson distribution for modeling the data. Since their search strategy is very specific to applications dealing with graph data, it is not applicable in the general case for multivariate continuous data dealt with in our work.

Regarding spatio-temporal data, Wu et al.\ \cite{wu2010spatio} follow a sequential approach for detecting anomalies first spatially, then temporally and apply a merge-strategy afterwards. However, the time needed for merging grows exponentially with the length of the time-series and their divergence measure is limited to binary-valued data. In contrast to this, our approach is able to deal with multivariate real-valued data efficiently and treats time and space jointly.

The remainder of this article is organized as follows: \Cref{sec:MDI} will introduce our novel ``Maximally Divergent Intervals'' algorithm for off-line detection of collective anomalies in multivariate spatio-temporal data. Its performance will be evaluated quantitatively on artificial data in \cref{sec:eval} and its suitability for practical applications will be demonstrated by means of experiments on real data from various different domains in \cref{sec:applications}. \Cref{sec:conclusions} will summarize the progress made so far and mention directions for future research.

\section{Maximally Divergent Intervals}
\label{sec:MDI}

\begin{figure}
    \includegraphics[width=\linewidth]{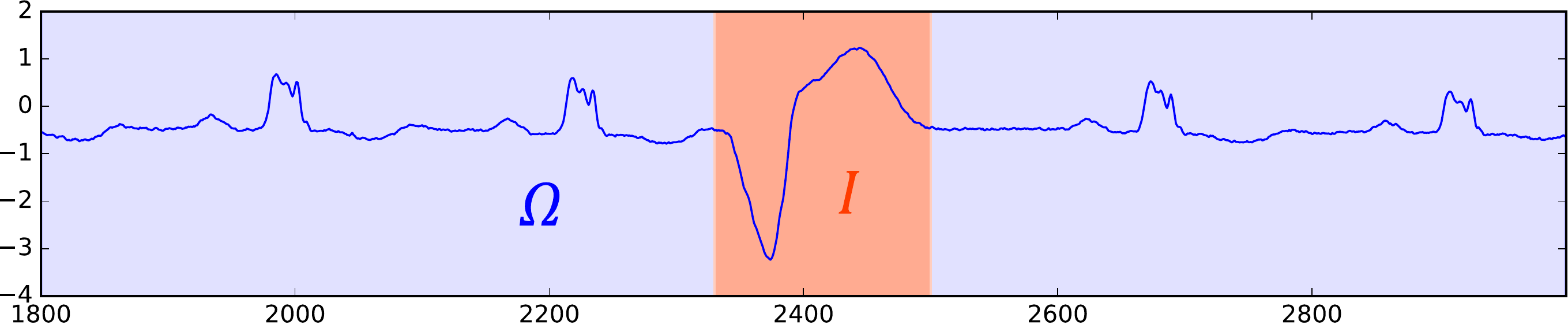}
    \caption{Schematic illustration of the principle of the MDI algorithm: The distribution of the data in the inner interval $I$ is compared with the distribution of the remaining time-series in the outer interval $\Omega$.}
    \label{fig:maxdiv-example}
\end{figure}

This section formally introduces our MDI algorithm for off-line detection of anomalous intervals in spatio-temporal data.
After a set of definitions that we are going to make use of, we start by giving a very rough overview of the basic idea behind the algorithm, which is also illustrated schematically in \cref{fig:maxdiv-example}. The subsequent sub-sections will go into more detail on the individual aspects and components of our approach.

Our implementation of the MDI algorithm is available as open source at: \url{https://cvjena.github.io/libmaxdiv/}

\subsection{Definitions}
\label{sec:MDI-defs}

Let $\mathfrak{X} \in \mathbb{R}^{T \times X \times Y \times Z \times D}$ be a multivariate spatio-temporal time-series given as \nth{5}-order tensor with 4 contextual attributes (point of time and spatial location) and $D$ behavioral attributes for all $N \coloneqq T \cdot X \cdot Y \cdot Z$ samples. We will index individual samples using 4-tuples $i \in \mathbb{N}^4$ like in $\mathfrak{X}_i \in \mathbb{R}^D$.

The usual interval notation $[\ell,r)$ will be used in the following for discrete intervals $\left\{ {t \in \mathbb{N} \vert \ell \le t < r} \right\}$. Furthermore, the set of all intervals with size between $a$ and $b$ along an axis of size $n$ is denoted by
\begin{equation}
	\mathfrak{I}_{a,b}^n \coloneqq
	\{ {
		[\ell,r) \mid
		%\ell,r \in \mathbb{N} \wedge
		1 \le \ell < r \le n+1 \wedge
		a \le r-\ell \le b
	} \} \;.
	\label{eq:interval-set-1d}
\end{equation}

The set of all sub-blocks of a data tensor $\mathfrak{X}$ complying with given size constraints $A = (a_t, a_x, a_y, a_z), B = (b_t, b_x, b_y, b_z)$ can then be defined as
\begin{equation}
\begin{split}
	\mathfrak{I}_{A,B} \coloneqq
	\{
        I_t \times I_x \times I_y \times I_z \mid
        & I_t \in \mathfrak{I}_{a_t,b_t}^T \wedge
        I_x \in \mathfrak{I}_{a_x,b_x}^X \wedge \\
        & I_y \in \mathfrak{I}_{a_y,b_y}^Y \wedge
        I_z \in \mathfrak{I}_{a_z,b_z}^Z
	\} \,.
    \label{eq:interval-set}
\end{split}
\end{equation}
In the following, we will often omit the indices for simplicity and just refer to it as $\mathfrak{I}$.

Given any sub-block $I \in \mathfrak{I}_{A,B}$, the remaining part of the time-series excluding that specific range can be defined as
\begin{equation}
    \Omega(I) \coloneqq
    \left( [1,T] \times [1,X] \times [1,Y] \times [1,Z] \right)
    \setminus I
    \label{eq:def-omega}
\end{equation}
and we will often simply refer to it as $\Omega$ if the corresponding range $I$ is obvious from the context.

\subsection{Idea and Algorithm Overview}
\label{sec:MDI-idea}

The approach pursued by the MDI algorithm to compute anomaly scores for all intervals $I \in \mathfrak{I}$ can be motivated by a long-standing definition of anomalies given by Douglas Hawkins \cite{hawkins1980} in 1980, who defines an anomaly as ``an observation which deviates so much from other observations as to arouse suspicions that it was generated by a different mechanism''.
In analogy to this definition, the MDI algorithm assumes that there is a sub-block $I \in \mathfrak{I}$ of the given time-series that has been generated according to ``a different mechanism'' than the rest of the time-series in $\Omega$ (cf. the schematic illustration in \cref{fig:maxdiv-example}). The algorithm tries to capture these mechanisms by modelling the probability density $p_I$ of the data in the inner interval $I$ and the distribution $p_\Omega$ in the outer interval $\Omega$. We investigate two different models for these distributions: Kernel Density Estimation (KDE) and multivariate normal distributions (Gaussians), which will be explained in detail in \cref{sec:MDI-density-estimation}.

Moreover, a measure $\mathfrak{D}(p_I, p_\Omega)$ for the degree of ``deviation'' of $p_I$ from $p_\Omega$ has to be defined. Like some other works on collective anomaly detection \cite{liu2013change,jiang2015general}, we use---among others---the \textit{Kullback-Leiber (KL) divergence} for this purpose. However, \cref{sec:divergences} will show that this is a sub-optimal choice when used without a slight modification and discuss alternative divergence measures.

Given these ingredients, the underlying optimization problem for finding the most anomalous interval can be described as
\begin{equation}
    \hat{I} = \argmax_{I \in \mathfrak{I}_{A,B}}\ \mathfrak{D}\left(p_I, p_{\Omega(I)}\right) \;.
    \label{eq:MDI-optimization-problem}
\end{equation}

Various possible choices for the divergence measure $\mathfrak{D}$ will be discussed in \cref{sec:divergences}.

In order to actually locate this ``maximally divergent interval'' $\hat{I}$, the MDI algorithm scans over all intervals $I \in \mathfrak{I}_{A,B}$, estimates the distributions $p_I$ and $p_\Omega$ and computes the divergence between them, which becomes the anomaly score of the interval $I$. The parameters $A$ and $B$, which define the minimum and the maximum size of the intervals in question, have to be specified by the user in advance. This is not a severe restriction, since extreme values may be chosen for these parameters in exchange for increased computation time. But depending on the application and the focus of the analysis, there is often prior knowledge about reasonable limits for the size of possible intervals.

After the anomaly scores have been obtained for all intervals, they are sorted in descending order and non-maximum suppression is applied to obtain non-overlapping intervals only. For large time-series with more than 10k samples, we apply an approximative non-maximum suppression that avoids storing all interval scores by maintaining a fixed-size list of currently best-scoring non-overlapping intervals.

Finally, the algorithm returns a ranking of intervals, so that a user-specified number of top $k$ intervals can be selected as output.

\subsection{Probability Density Estimation}
\label{sec:MDI-density-estimation}

The divergence measure used in \eqref{eq:MDI-optimization-problem} requires the notion of the distribution of the data in the intervals $I$ and $\Omega$. We will hence discuss in the following, which models we employ to estimate these distributions and how this can be done efficiently.

\subsubsection{Models}
\label{sec:MDI-distribution-models}

The choice of a specific model for the distributions $p_I$ and $p_\Omega$ imposes some assumptions about the data which may not conform to reality. However, since the MDI algorithm estimates the parameters of those distributions for all possible intervals in the time-series, the use of models that can be updated efficiently is crucial. One such model is Kernel Density Estimation (KDE) with
\begin{equation}
	p_\mathfrak{S}(\mathfrak{X}_i) = \frac{1}{\left| \mathfrak{S} \right|} \sum_{j \in \mathfrak{S}} k(\mathfrak{X}_i, \mathfrak{X}_j), \qquad \mathfrak{S} \in \left\{ {I, \Omega} \right\} ,
	\label{eq:KDE}
\end{equation}
using a Gaussian kernel
\begin{equation}
    k(x, y) = { \left( 2 \pi \sigma^2 \right) }^{- \frac{D}{2}} \cdot \exp \left( - \frac{\left\| x - y \right\|^2}{2 \sigma^2} \right) \;.
    \label{eq:gaussian-kernel}
\end{equation}
On the one hand, KDE is a very flexible model, but on the other hand, it does not scale well to long time-series and does not take correlations between attributes into account. The second proposed model does not expose these problems: It assumes that both the data in the anomalous interval $I$ and in the remaining time-series $\Omega$ are distributed according to multivariate normal distributions (\textit{Gaussians}) $\mathcal{N} \left( \mu_I, S_I \right)$ and $\mathcal{N} \left( \mu_\Omega, S_\Omega \right)$, respectively.

\begin{figure*}
    \includegraphics[width=\linewidth]{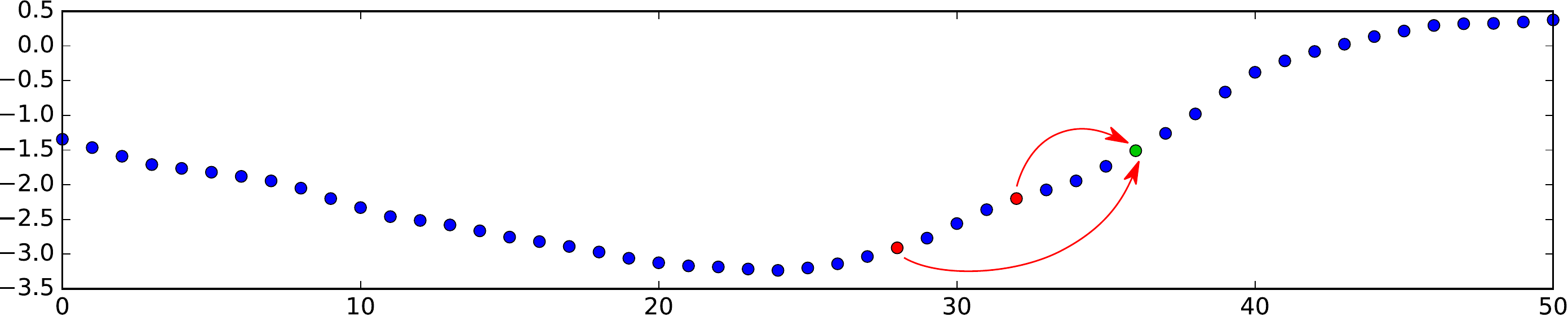}
    \caption{Illustration of time-delay embedding with $\kappa=3, \tau=4$. The attribute vector of each sample is augmented with the attributes of the samples 4 and 8 time steps earlier.}
    \label{fig:td-embedding}
\end{figure*}

\subsubsection{Efficient Estimation with Cumulative Sums}
\label{sec:MDI-cumsum}

Both distribution models described above involve a summation over all samples in the respective interval. Performing this summation for multiple intervals is redundant, because some of them overlap with each other. Such a naïve approach of finding the maximally divergent interval has a time complexity of $\mathcal{O} \left( N^2 \cdot L^2 \right)$ with KDE and $\mathcal{O} \left( N \cdot L \cdot \left( N + L \right) \right) \subseteq \mathcal{O} \left( N^2 \cdot L \right)$ with Gaussian distributions. This is due to the number of $\mathcal{O} \left( N \cdot L \right)$ intervals (with $L = (b_t - a_t + 1) \cdot (b_x - a_x + 1) \cdot (b_y - a_y + 1) \cdot (b_z - a_z + 1)$ being the maximum volume of an interval), each of them requiring a summation over $\mathcal{O} \left( L \right)$ samples for the evaluation of one of the divergence measures described later in \cref{sec:divergences}. For KDE, $\mathcal{O}(N)$ distance computations are necessary for the evaluation of the probability density function for each sample, while for Gaussian distributions a summation over all $\mathcal{O}(N)$ samples has to be performed for each interval to estimate the parameters of the distributions.

This would be clearly infeasible for large-scale data. However, these computations can be sped up significantly by using cumulative sums \cite{viola2004robust}. For the sake of clarity, we first consider the special case of a non-spatial time-series $(x_t)_{t=1}^n, x_t \in \mathbb{R}^D$. With regard to KDE, a matrix $C \in \mathbb{R}^{n \times n}$ of cumulative sums of kernelized distances can be used:

\begin{equation}
	C_{t,t'} = \sum_{t'' = 1}^{t'} k(x_t, x_{t''}) \;\;.
	\label{eq:cumsum-kde}
\end{equation}

This matrix has to be computed only once, which requires $\mathcal{O} \left( n^2 \right)$ distance calculations, and can then be used to estimate the probability density functions of the data in the intervals $I = \left[a,b\right)$ and $\Omega = \left[1,n\right] \setminus I$ in constant time:

\begin{equation}
\begin{split}
	p_I(x_t) &= \frac{C_{t,b-1} - C_{t,a-1}}{\left| I \right|} \;\; , \\[3ex]
	p_\Omega(x_t) &= \frac{C_{t,n} - C_{t,b-1} + C_{t,a-1}}{n - \left| I \right|} \;\; .
	\label{eq:cumsum-kde-recons}
\end{split}
\end{equation}

In analogy, a matrix $C^\mu \in \mathbb{R}^{D \times n}$ of cumulative sums over the samples and a tensor $C^S \in \mathbb{R}^{D \times D \times n}$ of cumulative sums over the outer products of the samples can be used to speed up the estimation of the parameters of Gaussian distributions:
\begin{equation}
	C_t^\mu = \sum_{t' = 1}^{t} x_{t'}, \quad
	C_t^S = \sum_{t' = 1}^{t} x_{t'} \cdot x_{t'}^\top \;,
\end{equation}
where $C_t^\mu$ and $C_t^S$ are the $t$-th column of $C^\mu$ and the $t$-th $D \times D$ matrix of $C^S$, respectively. Using these matrices, the mean vectors and covariance matrices can be estimated in constant time.

This technique can be generalized to the spatio-temporal scenario using higher order tensors for storing the cumulative sums. The reconstruction of a sum over a given range from such a cumulative tensor follows the \textit{Inclusion-Exclusion Principle} and the number of summands involved in the computation grows, thus, exponentially with the order of the tensor, being 16 for a \nth{4}-order tensor, compared to only 2 summands in the non-spatial case. The exact equation describing the reconstruction in the general case of an $M^\text{th}$-order tensor is given in \cref{app:cumsum-extraction}.

Thanks to the use of cumulative sums, the computational complexity of the MDI algorithm is reduced to $\mathcal{O} \left( N^2 + N \cdot L^2 \right)$  for the case of KDE and to $\mathcal{O} \left( N \cdot L^2 \right)$ for Gaussian distributions.

\subsection{Incorporation of Context}
\label{sec:MDI-embeddings}

The models used for probability density estimation described in the previous section are based on the assumption of independent samples. However, this assumption is almost never true for real data, since the value at a specific point of time and spatial location is likely to be strongly correlated with the values at previous times and nearby locations. To mitigate this issue, we apply two kinds of embeddings that incorporate context into each sample as pre-processing step.

\subsubsection{Time-Delay Embedding}
\label{sec:td-embedding}

Aiming to make combinations of observed values more representative of the hidden state of the system being observed, \textit{time-delay embedding} \cite{packard1980td} incorporates context from previous time-steps into each sample by transforming a given time-series $\left( x_t  \right)_{t=1}^n, x_t \in \mathbb{R}^D$, into another time-series $\left( x_t'  \right)_{t=1+(\kappa-1)\tau}^n, x_t' \in \mathbb{R}^{\kappa D}$, given by
\begin{equation}
	x_t' = \left(
		\begin{array}{ccccc}
		x_t^\top & x_{t-\tau}^\top & x_{t-2\tau}^\top
        & \cdots & x_{t-(\kappa-1) \cdot \tau}^\top
		\end{array}
	\right)^\top ,
    \label{eq:td-embedding}
\end{equation}
where the \textit{embedding dimension} $\kappa$ specifies the number of samples to stack together and the \textit{time lag} $\tau$ specifies the gap between two consecutive time-steps to be included as context. An illustrative example is given in \cref{fig:td-embedding}.

This method is often motivated by Takens' theorem \cite{takens1981detecting}, which, roughly, states that for a certain embedding dimension $\bar{\kappa}$ the hidden state of the system can be reconstructed given the observations of the last $\bar{\kappa}$ time-steps.

\subsubsection{Spatial-Neighbor Embedding}
\label{sec:spatial-embedding}

Correlations between nearby spatial locations are handled similarly: In addition to time-delay embedding, each sample of a spatio-temporal time-series can be augmented by the features of its spatial neighbors (cf. \cref{fig:spatial-embedding}) to enable the detection of spatial or spatio-temporal anomalies. This pre-processing step, which we refer to as \textit{spatial-neighbor embedding}, is parametrized with 3 parameters $\kappa_x, \kappa_y, \kappa_z$ for the embedding dimension along each spatial axis and 3 parameters $\tau_x, \tau_y, \tau_z$ for the lag along each axis.

\begin{figure*}
    \begin{subfigure}{0.32\linewidth}%
        \centering
        $\kappa_x = \kappa_y = 2, \quad \tau_x = \tau_y = 1$ \\
        \includegraphics[height=3.8cm]{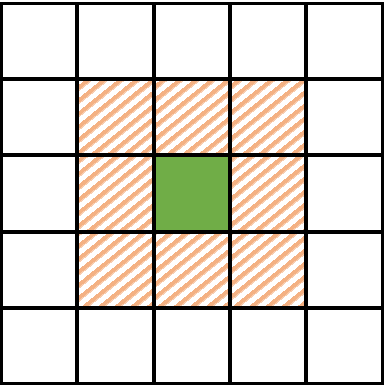}%
    \end{subfigure}%
    \begin{subfigure}{0.68\linewidth}%
        \centering
        $\kappa_x = 3, \quad \kappa_y = 2, \quad \tau_x = 3, \quad \tau_y = 2$ \\
        \includegraphics[height=3.8cm]{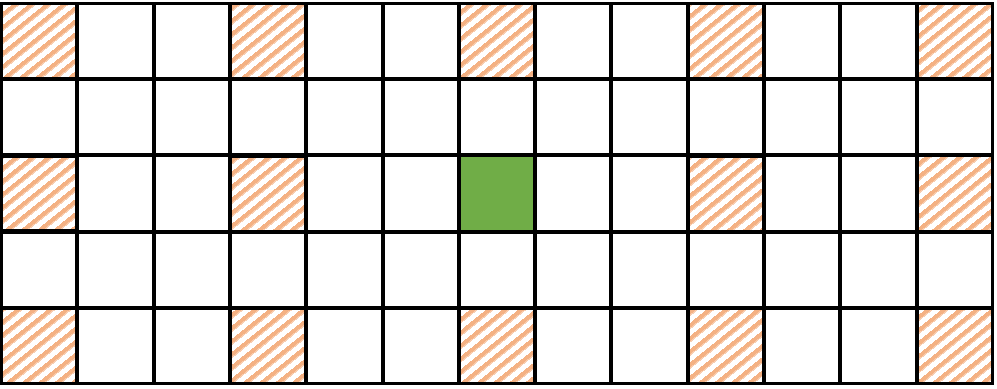}%
    \end{subfigure}%
    \caption{Exemplary illustration of spatial-neighbor embedding with different parameters. The attribute vector of the sample with a solid fill color is augmented with the attributes of the samples with a striped pattern.}
    \label{fig:spatial-embedding}
\end{figure*}

Note that, in contrast to time-delay embedding, neighbors from both directions are aggregated, since spatial context is bilinear. For example, $\kappa_x = 3$ would mean to consider 4 neighbors along the $x$-axis, 2 in each direction.

Spatial-neighbor embedding can either be applied before or after time-delay embedding. As opposed to many spatio-temporal anomaly detection approaches that perform temporal and spatial anomaly detection sequentially (e.g., \cite{wu2010spatio,kut2006spatio,cheng2006multiscale}), the MDI algorithm in combination with the two embeddings allows for a joint optimization. However, it implies a much more drastic multiplication of the data size.

\subsection{Divergences}
\label{sec:divergences}

A suitable measure for the deviation of the distribution $p_I$ from $p_\Omega$ is an essential part of the MDI algorithm. The following sub-sections introduce several divergence measures we have investigated and propose a modification to the well-known Kullback-Leibler (KL) divergence that is necessary for being able to compare divergences of distributions estimated from intervals of different size.

\subsubsection{Cross Entropy}
\label{sec:crossent}

Numerous divergence measures, including those described in the following, have been derived from the domain of \textit{information theory}. Being one of the most basic information theoretic concepts, the \textit{cross entropy} between two distributions given by their probability density functions $p$ and $q$ may already be used as a divergence measure:

\begin{equation}
    \divergence{CE}(p,q) \coloneqq \text{H}(p, q) \coloneqq \mathbb{E}_p \left[ - \log q \right] \;.
    \label{eq:crossent}
\end{equation}

%Obviously, entropy can be expressed as $\text{H}(p) = \text{H}(p, p)$.

Cross entropy measures how surprising a sample drawn from $p$ is, assuming that it would have been drawn from $q$, and is hence already eligible as a divergence measure, since the unexpectedness grows when $p$ and $q$ are very different.

Since the MDI algorithm assumes, that the data in the intervals $I \in \mathfrak{I}$ and $\Omega$ have been sampled from the distributions corresponding to $p_I$ and $p_\Omega$, respectively, the cross entropy of the two distributions can be approximated empirically from the data:

\begin{equation}
    \edivergence{CE}(I, \Omega) = \frac{1}{\left| I \right|} \sum_{i \in I} \log p_\Omega(\mathfrak{X}_i) \;.
    \label{eq:empirical-crossent}
\end{equation}

This approximation has the advantage of having to estimate only one probability density, $p_\Omega(x_t)$, explicitly. This is particularly beneficial, since the possibly anomalous intervals $I$ often contain only few samples, so that an accurate estimation of the probability density $p_I$ is difficult.

\subsubsection{Kullback-Leibler Divergence}
\label{sec:KL}

The \textit{Kullback-Leibler (KL) divergence} is a popular divergence measure that builds upon the fundamental concept of cross entropy. Given two distributions $p$ and $q$, the KL divergence can be defined as follows:
\begin{equation}
    \divergence{KL}(p, q) \coloneqq \text{H}(p, q) - \text{H}(p, p)
    = \mathbb{E}_p \left[ - \log \frac{p}{q} \right] \;.
    \label{eq:KL-def}
\end{equation}

As opposed to the pure cross entropy of $p$ and $q$, the KL divergence does not only take into account how well $p$ is explained by $q$, but also the intrinsic entropy $\text{H}(p,p) \eqqcolon \text{H}(p)$ of $p$, so that an interval with a stable distribution would get a higher score than an oscillating one if they had the same cross entropy with the rest of the time-series.

Like cross entropy, the KL divergence can be approximated empirically from the data, but in contrast to cross entropy, this requires estimating the probability densities of both distributions, $p_I$ and $p_\Omega$:
\begin{equation}
\begin{split}
    \edivergence{KL}(I,\Omega) &= \frac{1}{\left| I \right|} \cdot \sum_{i \in I} \log \left( \frac{p_I(\mathfrak{X}_i)}{p_\Omega(\mathfrak{X}_i)} \right) \\
    &= \frac{1}{\left| I \right|} \cdot \sum_{i \in I} { \log \left( p_I(\mathfrak{X}_i) \right) - \log \left( p_\Omega(\mathfrak{X}_i) \right) } \;.
    \label{eq:KL-IO}
\end{split}
\end{equation}

When used in combination with the Gaussian distribution model, the KL divergence comes with an additional advantage from a computational point of view, since there is a known closed-form solution for the KL divergence of two Gaussians \cite{duchi2007derivations}:
\begin{multline}
    \divergence{KL}\left( p_I, p_\Omega \right) = \frac{1}{2} \biggl(
    \left( \mu_\Omega - \mu_I \right)^\top S_{\Omega}^{-1} \left( \mu_\Omega - \mu_I \right) \\
    + \trace \left( S_{\Omega}^{-1} S_I \right)
    + \log \frac{ \left| S_\Omega \right| }{ \left| S_I \right| } - D
    \biggr) \;.
    \label{eq:KL-Gaussian-explicit}
\end{multline}

This allows evaluating the KL divergence in constant time for a given interval, which reduces the computational complexity of the MDI algorithm using the KL divergence in combination with Gaussian models to the number of possible intervals: $\mathcal{O} \left( N \cdot L \right)$.

Given this explicit solution for the KL divergence and the closed-form solution for the entropy of a Gaussian distribution \cite{ahmed1989entropy} with mean vector $\mu$ and covariance matrix $S$, which is given by
\begin{equation}
    \text{H}(\mathcal{N}(\mu, S)) = \frac{1}{2} \left( \log \left| S \right| + d + d \cdot \log \left( 2 \pi \right) \right) \;,
    \label{eq:gaussian-entropy}
\end{equation}
one can easily derive a closed-form solution for the cross entropy of those two distributions as well:

\begin{equation}
\begin{aligned}
    & \text{H}(p_I, p_\Omega) \\
    ={}& \divergence{KL}(p_I, p_\Omega) + \text{H}(p_I) \\
    ={}& \frac{1}{2} \biggl(
        \trace\left( S_\Omega^{-1} S_I \right)
        + \log\left| S_\Omega \right|
        + d \cdot \log(2 \pi) \\
    &{+}\:(\mu_\Omega - \mu_I)^\top S_\Omega^{-1} (\mu_\Omega - \mu_I)
    \biggr) \;.
\end{aligned}
\end{equation}

Compared with the KL divergence, this does not assign extremely high scores to small intervals $I$ with a low variance, due to the subtraction of $\log \left| S_I \right|$. This may be an explanation for the evaluation results in \cref{sec:eval}, where cross entropy in combination with Gaussian models is often superior to the KL divergence, although it does not account for intervals of varying entropy.

However, in contrast to the empirical approximation of cross entropy in \eqref{eq:empirical-crossent}, this requires the estimation of $p_I$.

\subsubsection{Polarity of the KL divergence and its effect on MDI}
\label{sec:KL-polarity}

It is worth noting that the KL divergence is not a metric and, in particular, not symmetric: $\divergence{KL}(p, q) \ne \divergence{KL}(q, p)$. Some authors use, thus, a symmetric variant \cite{liu2013change}:
\begin{equation}
    \divergence{KL-SYM}(p, q) = \frac{1}{2}\ \divergence{KL}(p, q) + \frac{1}{2}\ \divergence{KL}(q, p) \;.
    \label{eq:KL-sym}
\end{equation}

This raises the question whether $\divergence{KL}(p_I, p_\Omega)$, $\divergence{KL}(p_\Omega, p_I)$, or the symmetric version $\divergence{KL-SYM}$ should be used for the detection of anomalous intervals. Quantitative experiments with an early prototype of our method \cite{Rodner16:MDI} have shown that neither $\divergence{KL}(p_\Omega, p_I)$ nor $\divergence{KL-SYM}$ provide good performance, as opposed to $\divergence{KL}(p_I, p_\Omega)$.

A visual inspection of the detections resulting from the use of $\divergence{KL}(p_\Omega, p_I)$ with the assumption of Gaussian distributions shows that all the intervals with the highest anomaly scores have the minimum possible size specified by the user and a very low variance. An example is given in \cref{fig:omega-i-bias}. The scores of the top detections in that example are around 100 times higher than those yielded by $\divergence{KL}(p_I, p_\Omega)$.

\begin{figure}[b]
    \includegraphics[width=\linewidth]{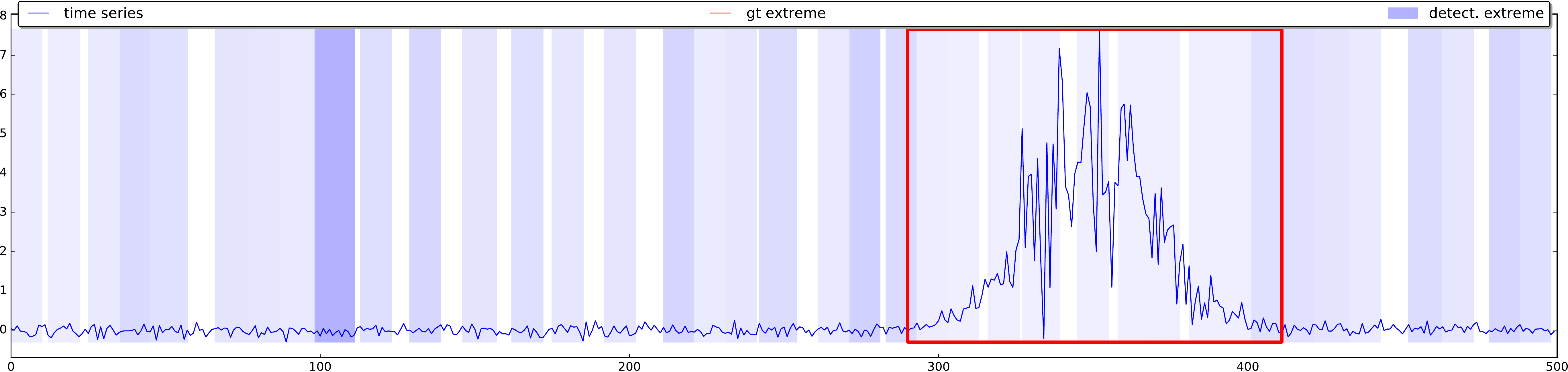}
    \caption{Example for the bias of $\divergence{KL}(p_\Omega, p_I)$ detections towards small intervals with low empirical variance on a synthetic time-series. The intensity of the fill color of the detected intervals corresponds to the detection scores. The ground-truth anomalous interval is indicated by a red box.}
    \label{fig:omega-i-bias}
\end{figure}

This bias of $\divergence{KL}(p_\Omega, p_I)$ towards small low-variance intervals can also be explained theoretically. For the sake of simplicity, consider the special case of a univariate time-series. In this case, the closed-form solution for $\divergence{KL}(p_\Omega, p_I)$ assuming Gaussian distributions given in \eqref{eq:KL-Gaussian-explicit} reduces to
\begin{equation}
    \frac{1}{2} \left(
        \frac{\sigma_\Omega^2}{\sigma_I^2}
        + \frac{(\mu_I - \mu_\Omega)^2}{\sigma_I^2}
        + \log \sigma_I^2 - \log \sigma_\Omega^2 - 1
    \right) \;,
    \label{eq:kl-oi-gaussian-univar}
\end{equation}
where $\mu_I$, $\mu_\Omega$ are the mean values and $\sigma_I^2$, $\sigma_\Omega^2$ are the variances of the distributions in the inner and in the outer interval, respectively. It can be seen from \eqref{eq:kl-oi-gaussian-univar} that, due to the division by $\sigma_I^2$, the KL divergence will approach infinity when the variance in the inner interval converges towards $0$. And since the algorithm has to estimate the variance empirically from the given data, it assigns high detection scores to intervals as small as possible, because smaller intervals have a higher chance of having a low empirical variance. The term $\log \sigma_I^2$ cannot counterbalance this effect, though it is negative for $\sigma_I < 1$, since its absolute value grows much more slowly than that of $\sigma_I^{-2}$, as can be seen from the fact that $\forall_{\sigma_I < 1}\left( - \log \sigma_I^2 = \log \sigma_I^{-2} < \sigma_I^{-2} \right)$, since $\forall_{\sigma_I < 1}\left( \sigma_I^{-2} > 1 \right)$.

In contrast, $\divergence{KL}(p_I, p_\Omega)$, where the roles of $I$ and $\Omega$ are swapped, does not possess this deficiency, since $\sigma_\Omega^2$ is estimated from a much larger portion of data and, thus, is a more robust estimate.

The symmetric version $\divergence{KL-SYM}(p_I, p_\Omega)$ is useless as well, since the scores obtained from $\divergence{KL}(p_I, p_\Omega)$ will just be absorbed by the much higher scores of $\divergence{KL}(p_\Omega, p_I)$.

\subsubsection{Statistical Analysis and Unbiased KL Divergence}
\label{sec:KL-unbiased}

Though $\divergence{KL}(p_I, p_\Omega)$ does not overestimate the anomalousness of low-variance intervals as extremely as $\divergence{KL}(p_\Omega, p_I)$ does, the following theoretical analysis will show that it is not unbiased either. In contrast to the previous section, this bias is not related to the data itself, but to the length of the intervals: smaller intervals systematically get higher scores than longer ones. This harms the quality of interval detections, because anomalies will be split up into multiple contiguous small detections (see \cref{fig:kl-det-albedo-biased} for an example).

\begin{figure*}
    \begin{subfigure}{0.49\linewidth}%
        \caption{$\divergence{KL}(p_I, p_\Omega)$}%
        \includegraphics[width=\linewidth]{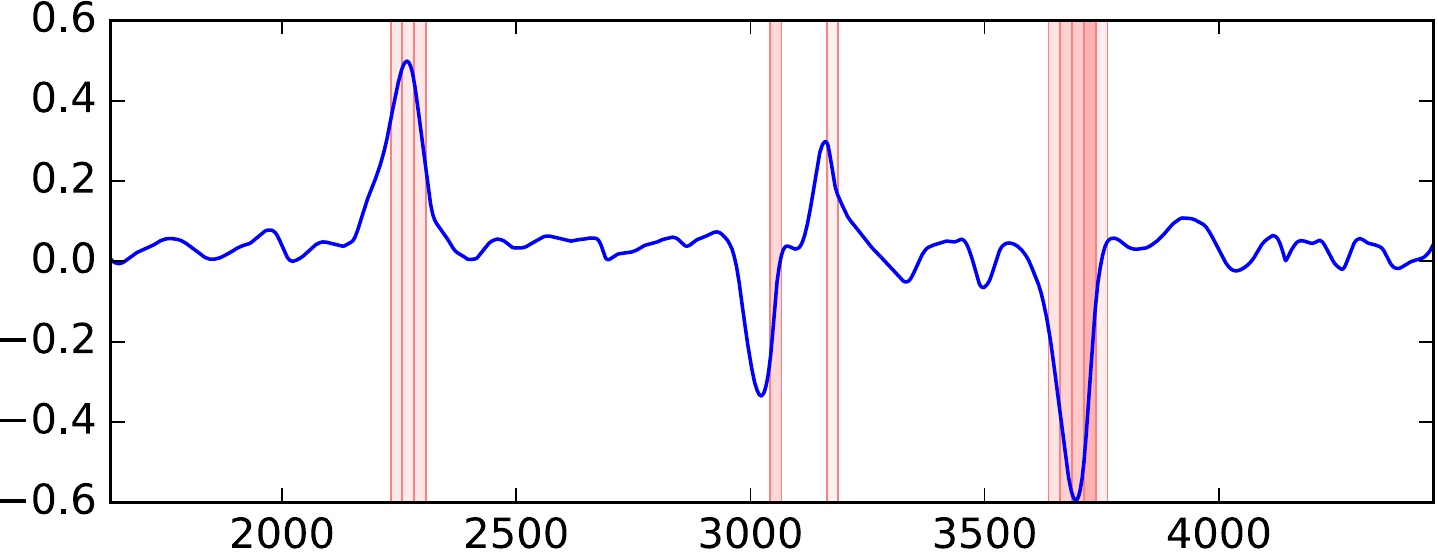}%
        \label{fig:kl-det-albedo-biased}%
    \end{subfigure}%
    \hfill%
    \begin{subfigure}{0.49\linewidth}%
        \caption{$\divergence{U-KL}(p_I, p_\Omega)$}%
        \includegraphics[width=\linewidth]{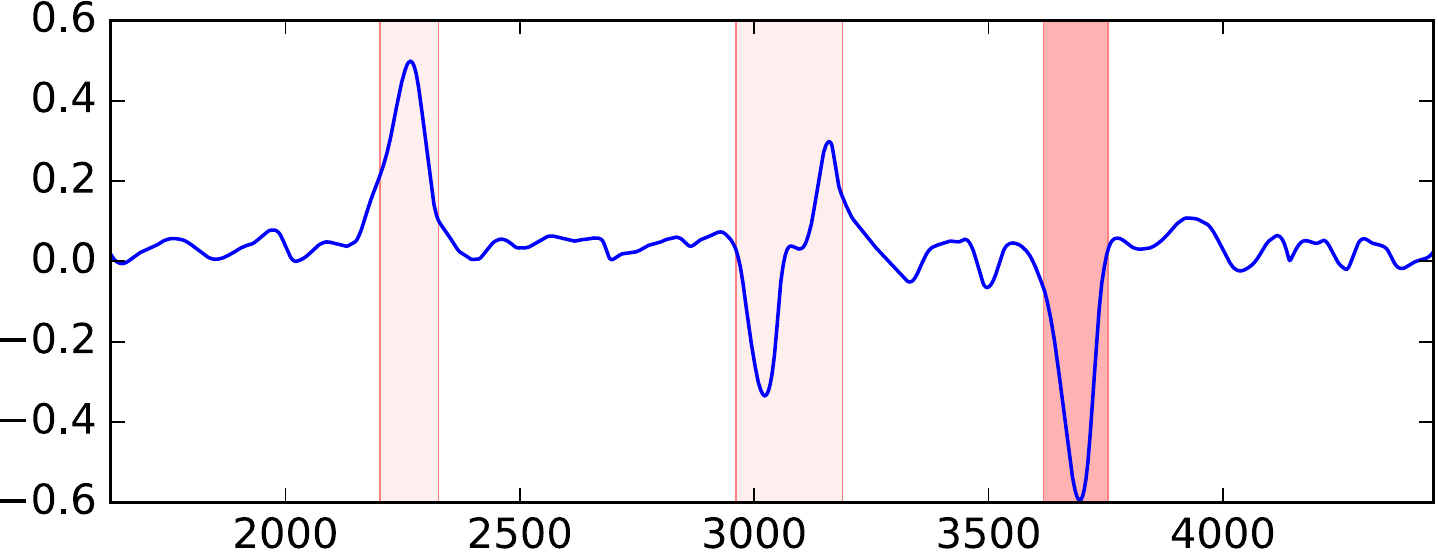}%
        \label{fig:kl-det-albedo-unbiased}%
    \end{subfigure}%
    \caption{(\subref{fig:kl-det-albedo-biased}) Top 10 detections obtained from the KL divergence on a real time-series and  (\subref{fig:kl-det-albedo-unbiased}) top 3 detections obtained from the unbiased KL divergence on the same time-series. This example illustrates the phenomenon of several contiguous minimum-size detections when using the original KL divergence (note the thin lines between the single detections in the left plot). The MDI algorithm has been applied with a time-delay embedding of $\kappa=3, \tau=1$ and the size of the intervals to analyze has been limited to be between 25 and 250 samples.}
    \label{fig:kl-unbiased-comparison}
\end{figure*}

Recall that $\mathfrak{I}_{m,m}^n$ denotes the set of all intervals of length $m$ in a time-series with $n$ time-steps. Furthermore, let $\vec{0}^d, d \in \mathbb{N},$ denote a $d$-dimensional vector with all coefficients being 0 and $\mathbb{I}_d$ the identity matrix of dimensionality $d$.

When applying the MDI algorithm to a time-series $(x_t)_{t=1}^n, x_t \sim \mathcal{N}(\vec{0}^d, \mathbb{I}_d)$, sampled independently and identically from plain white noise, an ideal divergence is supposed to yield constant average scores for all $\mathfrak{I}_{m,m}, m = a, \dots, b$ (for some user-defined limits $a,b$), i.e., scores independent from the length of the intervals.

For simplicity, we will first analyze the distribution of those scores using the MDI algorithm with Gaussian distributions with the simple, but for this data perfectly valid assumption of identity covariance matrices. In this case, the KL divergence $\divergence{KL}(p_I, p_\Omega)$ of two Gaussian distributions with the mean vectors $\mu_I, \mu_\Omega \in \mathbb{R}^d$ in some intervals $I \in \mathfrak{I}_m, \Omega = [1,n] \setminus I$ for some arbitrary $m$ is given by $\frac{1}{2} \left\| \mu_\Omega - \mu_I \right\|^2$. Moreover, since all samples in the time-series are normally distributed, so are their empirical means:
\begin{align*}
    \mu_I &= \frac{1}{m} \sum_{t \in I} x_t \sim \mathcal{N}(\vec{0}^d, m^{-1} \cdot \mathbb{I}_d) \;, \\[3ex]
    \mu_\Omega &= \frac{1}{n - m} \sum_{t \notin I} x_t \sim \mathcal{N}(\vec{0}^d, (n-m)^{-1} \cdot \mathbb{I}_d) \;.
\end{align*}

Thus, all dimensions of the mean vectors are independent and identically normally distributed variables. Their difference is, hence, normally distributed too:

\begin{equation*}
    \mu_\Omega - \mu_I \sim \mathcal{N} \left( \vec{0}^d, \left(\frac{1}{m} + \frac{1}{n-m}\right) \cdot \mathbb{I}_d \right) \;.
\end{equation*}

Thus, $(\mu_\Omega - \mu_I) / \sqrt{\frac{1}{m} + \frac{1}{n-m}} \sim \mathcal{N}(\vec{0}^d, \mathbb{I}_d)$ is a vector of independent standard normal random variables and
\begin{equation}
\begin{split}
    & \divergence{KL}(p_I, p_\Omega) \\
    ={}& \frac{1}{2} \left(\frac{1}{m} + \frac{1}{n-m}\right)
    \sum_{i=1}^{d} \left( \frac{(\mu_\Omega - \mu_I)_i}{\sqrt{\frac{1}{m} + \frac{1}{n-m}}} \right)^2 \\
    \sim{}& \frac{1}{2} \left(\frac{1}{m} + \frac{1}{n-m}\right) \cdot \chi_d^2
    \label{eq:kl-distribution}
\end{split}
\end{equation}
is the sum of the squares of $d$ independent normal variables and, hence, distributed according to the chi-squared distribution with $d$ degrees of freedom, scaled by half the variance of the variables. The mean of a $\chi_d^2$-distributed random variable is $d$ and the mean of the $\divergence{KL}(p_I, p_\Omega)$ scores for all intervals in $\mathfrak{I}_m$ is, accordingly, $\frac{d}{2} \left( \frac{1}{m} + \frac{1}{n-m} \right)$, which is inversely proportional to the length of the interval $m$. Thus, the KL divergence is systematically biased towards smaller intervals.

When the length $n$ of the time-series is very large, the asymptotic scale of the chi-squared distribution is $\lim\limits_{n \rightarrow \infty} \frac{1}{2} \left(\frac{1}{m} + \frac{1}{n-m}\right) = \frac{1}{2m}$ and the estimated parameters $\mu_\Omega, S_\Omega$ of the outer distribution converge towards the parameters of the true distribution of the data. Thus, if the restriction of the Gaussian model to identity covariance matrices is weakened to a global, shared covariance matrix $S$, the above findings also apply to the case of long time-series with correlated variables and, hence, also when time-delay embedding is applied. Because in this case, the KL divergence reduces to $\frac{1}{2} (\mu_I - \mu_\Omega)^\top S^{-1} (\mu_I - \mu_\Omega)$ and the subtraction of the true mean $\mu_\Omega$ followed by the multiplication with the inverse covariance matrix can be considered as a normalization of the time-series, transforming it to standard normal variables with uncorrelated dimensions.

For the general case of two unrestricted Gaussian distributions, the test statistic
\begin{multline}
    \lambda \coloneqq\
    dm (\log(m) - 1)
    + m (\mu_I - \mu_\Omega)^\top S_\Omega^{-1} (\mu_I - \mu_\Omega) \\
    + \trace\left( m S_I S_\Omega^{-1} \right)
    - m \cdot \log \left| m S_I S_\Omega^{-1} \right|
\label{eq:test-statistic}
\end{multline}
has been shown to be asymptotically distributed according to a chi-squared distribution with $d + \frac{d(d+1)}{2}$ degrees of freedom \cite{anderson1962mvstat}. This test statistic is often used for testing the hypothesis that a given set of samples has been drawn from a Gaussian distribution with known parameters \cite{kanungo1995mvhypot}. In the scenario of the MDI algorithm, the set of samples is the data in the inner interval $I$ and the parameters of the distribution to test that data against are those estimated from the data in the outer interval $\Omega$. The null hypothesis of the test would be that the data in $I$ has been sampled from the same distribution as the data in $\Omega$. The test statistic may then be used as a measure for how well the data in the interval $I$ fit the model established based on the data in the remainder of the time-series.

After some elementary reformulations, the relationship between this test statistic $\lambda$ and the KL divergence becomes obvious: $\lambda = 2m \cdot \divergence{KL}(p_I, p_\Omega)$. This is exactly the normalization of the KL divergence by the scale factor identified in \eqref{eq:kl-distribution}. Thus, we define an {\em unbiased KL divergence} as follows:

\begin{equation}
    \divergence{U-KL}(p_I, p_\Omega) \coloneqq
    2 \cdot \left| I \right| \cdot \divergence{KL}(p_I, p_\Omega) \;.
    \label{eq:kl-unbiased}
\end{equation}

The distribution of this divergence applied to asymptotically long time-series depends only on the number $d$ of attributes and not on the length $m$ of the interval any more. However, this correction may also be useful for time-series of finite length. An example of actual detections resulting from the use of the unbiased KL divergence compared with the original one can be seen in \cref{fig:kl-unbiased-comparison}.

A further advantage of knowing the distribution of the scores is that this knowledge can also be used for normalizing the scores with respect to the number of attributes, in order to make them comparable across time-series of varying dimensionality.
Moreover, it allows the selection of a threshold for distinguishing between anomalous and nominal intervals based on a chosen significance level. This may be preferred in some applications over searching for a fixed number of top $k$ detections.

Interestingly, Jiang et al.\ \cite{jiang2015general} have derived an equivalent unbiased KL divergence ($m \cdot \divergence{KL}(p_I, p_\Omega)$) from a different starting point based on the assumption of a Poisson distribution and the inverse log-likelihood of the interval as anomaly score.

\begin{figure*}
    \begin{subfigure}{0.49\linewidth}%
        \caption{amplitude\_change\_multvar}%
        \includegraphics[width=\linewidth]{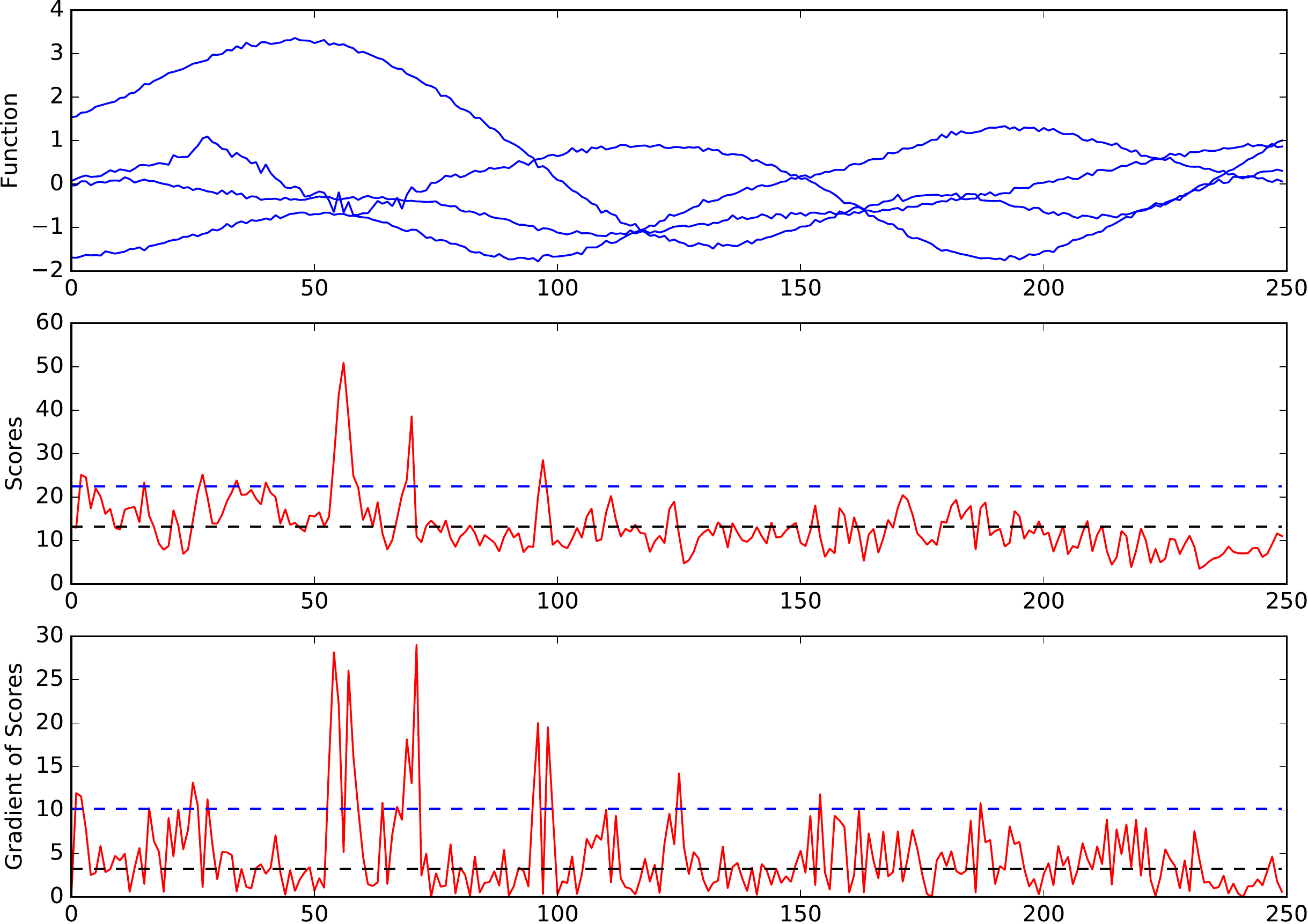}%
    \end{subfigure}%
    \hfill%
    \begin{subfigure}{0.49\linewidth}%
        \caption{frequency\_change\_multvar}%
        \includegraphics[width=\linewidth]{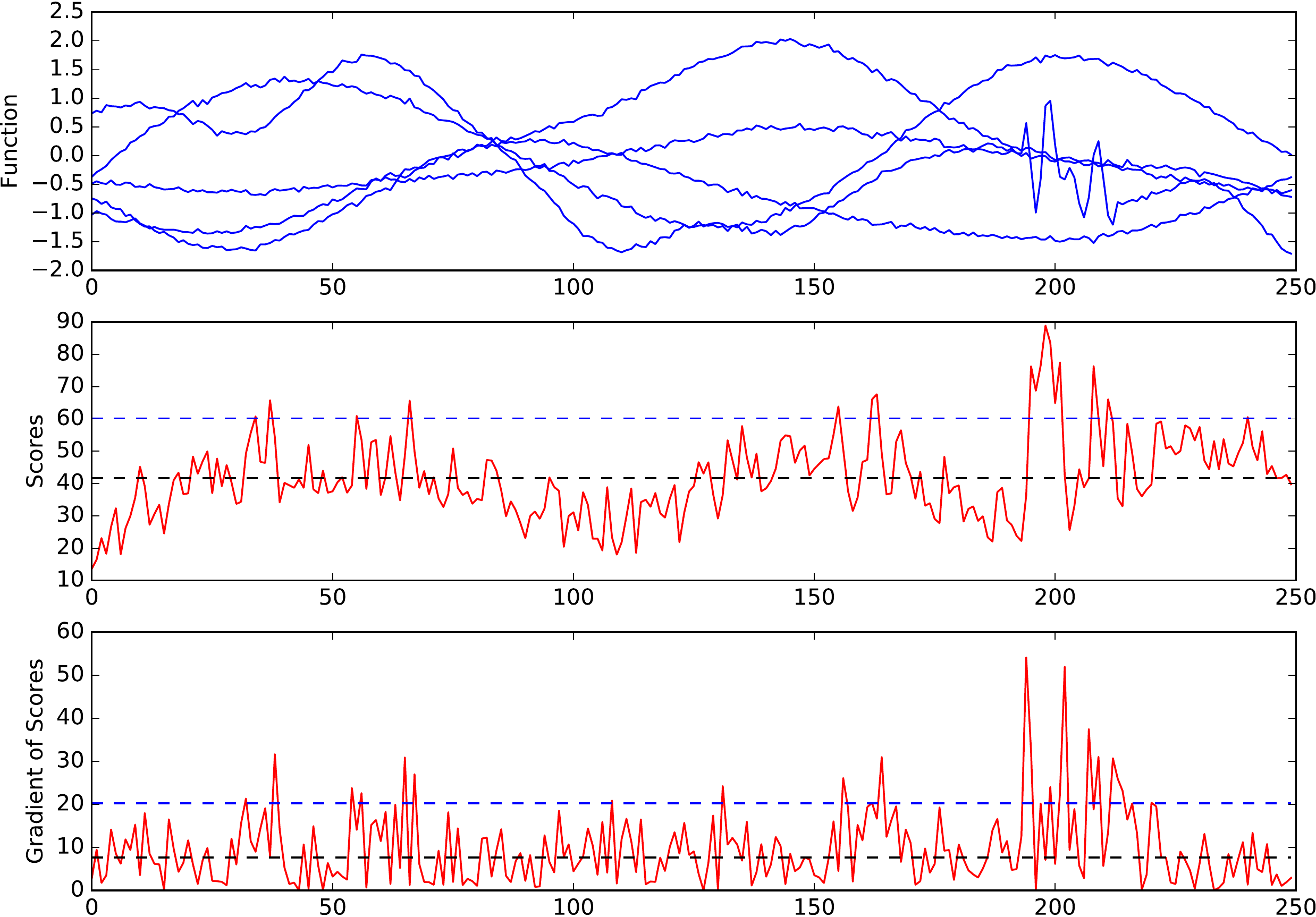}%
    \end{subfigure}%
    \caption{Two exemplary synthetic time-series along with the corresponding Hotelling's $T^2$ scores and their gradients. The dashed black line indicates the mean of the scores and the dashed blue line marks a threshold that is 1.5 standard deviations above the mean. Time-delay embedding with $\kappa=3, \tau=1$ was applied before computing the scores.}
    \label{fig:proposal-examples}
\end{figure*}

\subsubsection{Jensen-Shannon Divergence}
\label{sec:js-divergence}

A divergence measure that does not expose the problem of being asymmetric is the \textit{Jensen-Shannon (JS) divergence}, which builds upon the KL divergence:
\begin{equation}
    \divergence{JS}(p, q)
    = \frac{1}{2}\ \divergence{KL} \left( p, \frac{p+q}{2} \right)
    + \frac{1}{2}\ \divergence{KL} \left( q, \frac{p+q}{2} \right) \;.
    \label{eq:js-def}
\end{equation}
where $p$ and $q$ are probability density functions. $\frac{p+q}{2}$ is a mixture distribution, so that a sample is drawn either from $p$ or from $q$ with equal probability (though a parametrized version of the JS divergence accounting for unequal prior probabilities exists as well, but will not be covered here).

The JS divergence possesses some desirable properties, which the KL divergence does not have: most notably, it is symmetric and bounded between $0$ and $\log 2$ \cite{lin1991jsdivergence}, so that anomaly scores cannot get infinitely high.

Like the KL divergence, the JS divergence can be approximated empirically from the data in the intervals $I$ and $\Omega$.
However, there is no closed-form solution for the JS divergence under the assumption of a Gaussian distribution (as opposed to the KL divergence), since $\frac{p_I + p_\Omega}{2}$ would then be a Gaussian Mixture Model (GMM). Though several approximations of the KL divergence of GMMs have been proposed, they are either computationally expensive or abandon essential properties such as positivity \cite{hershey2007approximating}. This lack of a closed-form solution is likely to be the reason why the JS divergence was clearly outperformed by the KL divergence in our quantitative experiments in \cref{sec:eval} when the Gaussian model is used, despite its desirable theoretic properties.

\subsection{Interval Proposals for Large-Scale Data}
\label{sec:proposals}

Exploiting cumulative sums and a closed-form solution for the KL divergence, the asymptotic time complexity of the MDI algorithm with a Gaussian distribution model could already be reduced to be linear in the number of intervals (see \cref{sec:MDI-cumsum}). If the maximum length of an anomalous interval is independent from the number of samples $N$, the run-time is also linear in $N$. However, due to high constant-time requirements for estimating probability densities and computing the divergence, the algorithm is still too slow for processing large-scale data sets with millions of samples.

Since anomalies are rare by definition, many of the intervals analyzed by a full scan will be uninteresting and irrelevant for the list of the top anomalies detected by the algorithm.
In order to focus on the analysis of non-trivial intervals, we employ a simple proposal technique that selects interesting intervals based on point-wise anomaly scores.

Simply grouping contiguous detections of point-wise anomaly detection methods in order to retrieve anomalous intervals is insufficient, because it will most likely lead to split-up detections. However, it is not unreasonable to assume that many samples inside of an anomalous interval will also have a high point-wise score, especially after applying contextual embedding. \Cref{fig:proposal-examples}, for example, shows two exemplary time-series from the synthetic data set introduced in \cref{sec:eval-dataset} along with the point-wise scores retrieved by applying the Hotelling's $T^2$ method \cite{macgregor1995statistical}, after time-delay embedding has been applied to the time-series. Note that even in the case of the very subtle amplitude-change anomaly, the two highest Hotelling's $T^2$ scores are at the beginning and the end of the anomaly. The idea is to apply a simple threshold operation on the point-wise scores to extract interesting points and then propose all those intervals for detailed scoring by a divergence measure whose first and last samples are among these points if the interval conforms to the size constraints.

This way, the probability density estimation and the computation of the divergence have to be performed for a comparatively small set of interesting intervals only and not for all possible intervals in the time-series. The interval proposal method is not required to have a low false-positive rate, though, because the divergence measure is responsible for the actual scoring. Instead, it has to act as a high-recall system so that truly anomalous intervals are not excluded from the actual analysis.

Since we are only interested in the beginning and end of the anomalies, the point-wise scores are not used directly, but the centralized gradient filter $\left[-1 \quad 0 \quad 1 \right]$ is applied to the scores for reducing them in areas of constant anomalousness and emphasizing changes of the anomaly scores.

The evaluation in \cref{sec:eval-proposals} will show that the interval proposal technique can speed-up the MDI algorithm significantly without impairing its performance.

\section{Experimental Evaluation}
\label{sec:eval}

In this section, we evaluate our MDI algorithm on a quantitative basis using synthetic data and compare it with other approaches well-known in the field of anomaly detection.

\subsection{Data Set}
\label{sec:eval-dataset}

In contrast to many other established machine learning tasks, there is no widely used standard benchmark for the evaluation of anomaly detection algorithms; not for the detection of anomalous intervals and not even for the very common task of point-wise anomaly detection. This is mainly for the reason that the notion of an ``anomaly'' is not well defined and varies between different applications and even from analyst to analyst. Moreover, anomalies are, by definition, rare, which makes the collection of large-scale data sets difficult. However, even if a large amount of data were available, it would be nearly impossible to annotate it in an intersubjective way everyone would agree with. But accurate and complete ground-truth information is mandatory for a quantitative evaluation and comparison of machine learning techniques. Therefore, we use a synthetic data set for assessing the performance of different variants of the MDI algorithm.

All time-series in that data set have been sampled from a Gaussian process $\mathcal{GP}(m, K)$ with a squared-exponential covariance function $K(x_t,x_{t'}) = \left( 2 \pi \ell^2 \right)^{-\sfrac{1}{2}} \cdot \exp\left( - \frac{\left\| x_t - x_{t'} \right\|^2}{2\ell^2} \right) + \sigma^2 \cdot \delta(t,t')$ and zero mean function $m(x) = 0$. The \textit{length scale} of the GP has been set to $\ell^2 = 0.01$ and the noise parameter to $\sigma^2 = 0.001$. $\delta(t,t')$ denotes Kronecker's delta. Different types of anomalies have then been injected into these time-series, with a size varying between 5\% and 20\% of the length of the time-series:
\\

\noindent
\textbf{meanshift:}
A random, but constant value $\gamma \in [3,4]$ is added to or subtracted from the anomalous samples.

\noindent
\textbf{meanshift\_hard:}
A random, but constant value $\gamma \in [0.5,1]$ is added to or subtracted from the anomalous samples.

\noindent
\textbf{meanshift5:}
Five \texttt{meanshift} anomalies are inserted into the  time-series.

\noindent
\textbf{meanshift5\_hard:}
Five \texttt{meanshift\_hard} anomalies inserted into the time-series.

\noindent
\textbf{amplitude\_change:}
The time-series is multiplied with a Gaussian window with standard deviation $\sfrac{L}{4}$ whose mean is the centre of the anomalous interval. Here, $L$ is the length of the anomalous interval and the amplitude of the Gaussian window is clipped at $2.0$. This modified time-series is added to the original one.

\noindent
\textbf{frequency\_change:}
The time-series is sampled from a non-stationary GP, whose covariance function $K(x_t,x_{t'}) = \left( \ell^2(t) \cdot \ell^2(t') \right)^{\sfrac{1}{4}} \cdot \left( \frac{\ell^2(t) + \ell^2(t')}{2} \right)^{-\sfrac{1}{2}} \cdot \exp\left(- \frac{\left\| x_t - x_{t'} \right\|^2}{\ell^2(t) + \ell^2(t')} \right) + \sigma \cdot \delta(t,t')$ uses a reduced length scale $\ell^2(t) = \left\{\begin{matrix}
    10^{-2} & \text{if}\ t \notin [a,b), \\ 
    10^{-4} & \text{if}\ t \in [a,b)\phantom{,}
\end{matrix}\right.$ during the anomalous interval $I = [a,b)$, so that correlations between samples are reduced, which leads to more frequent oscillations \cite{paciorek2004nonstationary}.

\noindent
\textbf{mixed:}
The values in the anomalous interval are replaced with the values of another function sampled from the Gaussian process. 10 time-steps at the borders of the anomaly are interpolated between the two functions for a smooth transition. This rather difficult test case is supposed to reflect the concept of anomalies as being ``generated by a different mechanism'' (cf. \cref{sec:MDI-idea}).
\\

\begin{sloppypar}
The above test cases are all univariate, but there are as well similar multivariate scenarios \texttt{meanshift\_multvar}, \texttt{amplitude\_change\_multvar}, \texttt{frequency\_change\_multvar}, and \texttt{mixed\_multvar} with 5-dimensional time-series. Regarding the first three of these test cases, the corresponding anomaly is injected into one of the dimensions, while all attributes are replaced with those of the other time-series in the \texttt{mixed\_multvar} scenario, which is also a property of many real time-series.
\end{sloppypar}

This results in a synthetic test data set with 11 test cases, a total of 1100 time-series and an overall number of 1900 anomalies.
Examples for all test cases are shown in \cref{fig:synthetic-examples}.

\begin{figure}[tb]
    \includegraphics[width=\linewidth]{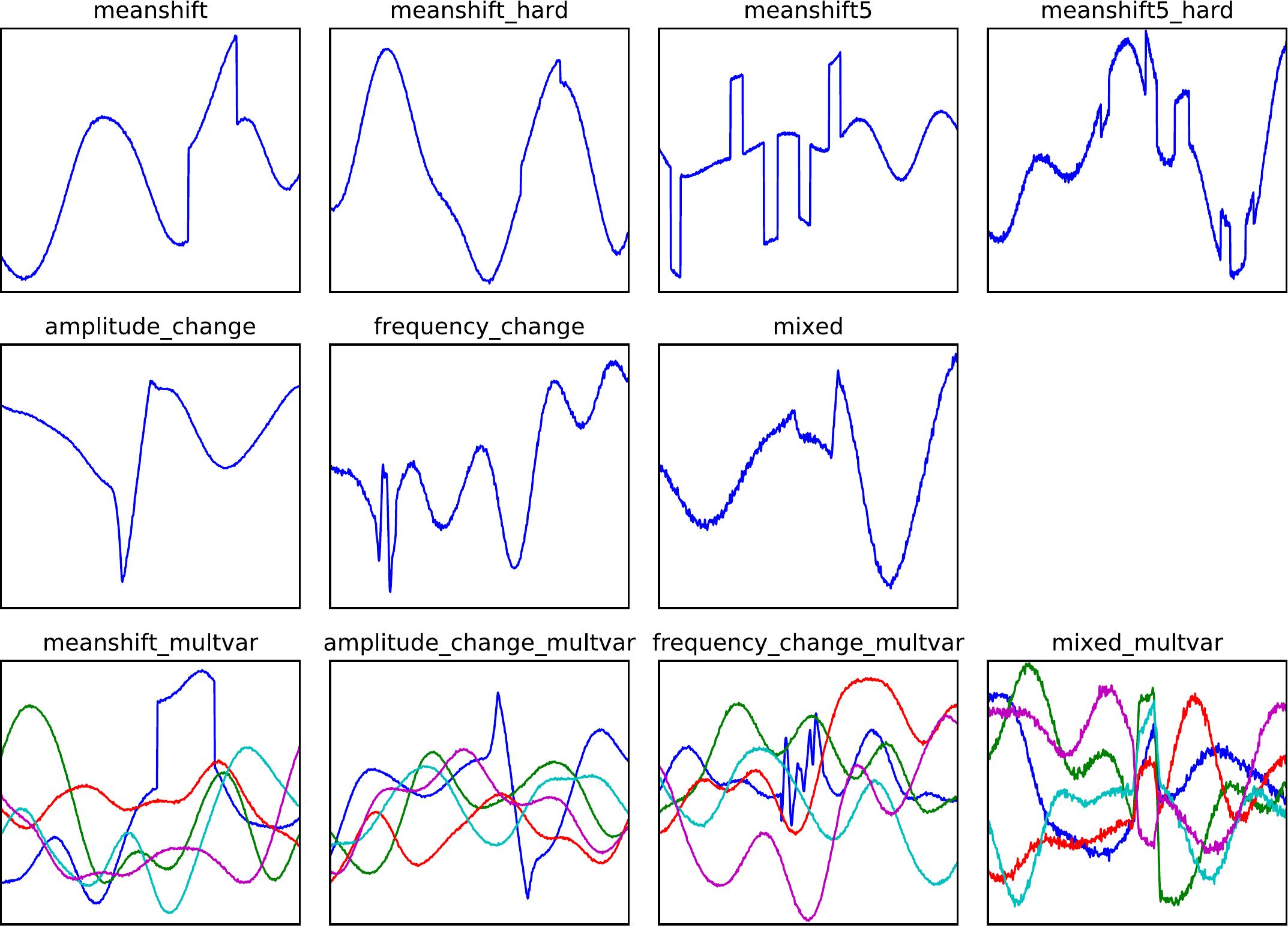}
    \caption{Examples from the synthetic test data set.}
    \label{fig:synthetic-examples}
\end{figure}

\subsection{Performance Comparison}
\label{sec:eval-performance}

\begin{figure*}
    \includegraphics[width=\linewidth]{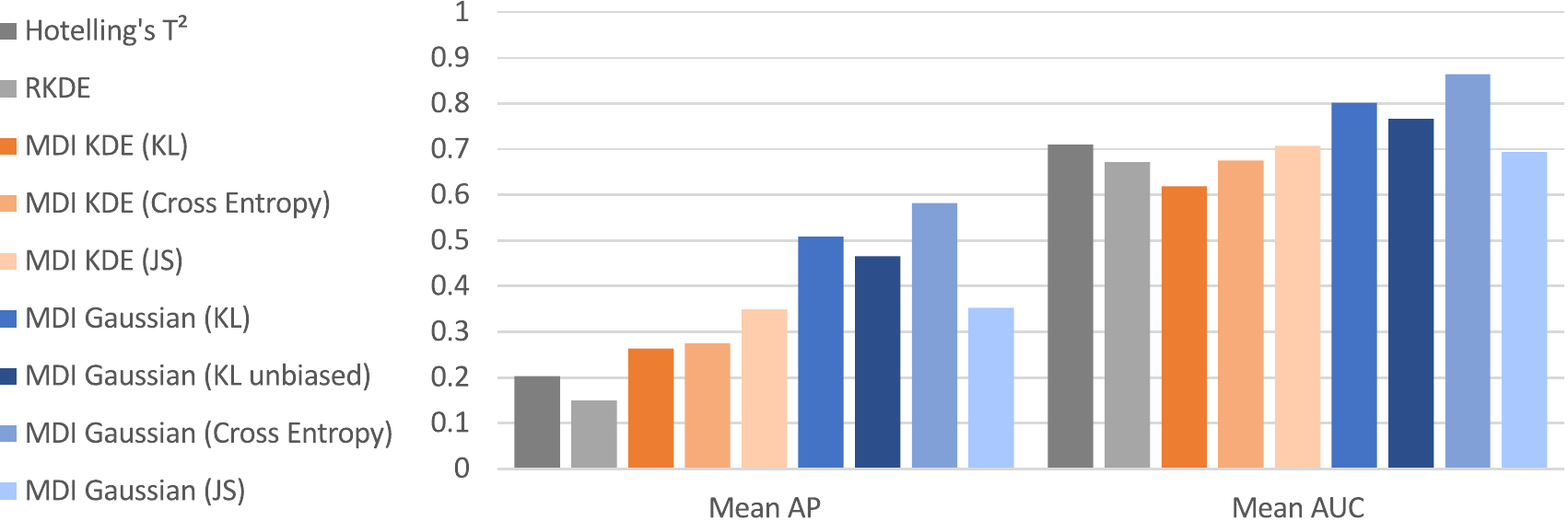}
    \caption{Performance comparison of different variants of the MDI algorithm and the baselines on the synthetic data set.}
    \label{fig:performance-comparison}
\end{figure*}

\begin{figure*}
    \begin{minipage}{0.54\linewidth}
        \includegraphics[width=\linewidth]{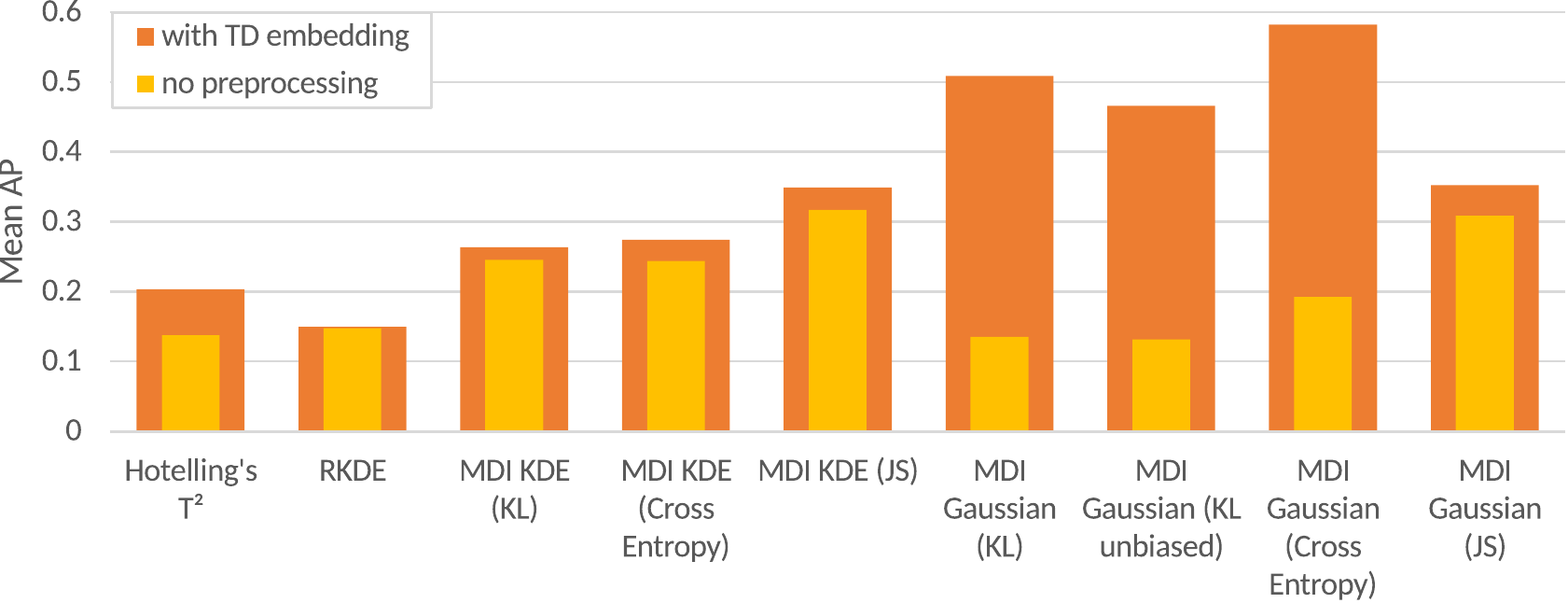}
        \captionof{figure}{Effect of time-delay embedding with $\kappa=6, \tau=2$ on the performance of the MDI algorithm and the baselines on the synthetic data set.}
        \label{fig:td-effect}
    \end{minipage}
    \hfill
    \begin{minipage}{0.41\linewidth}
        \includegraphics[width=\linewidth]{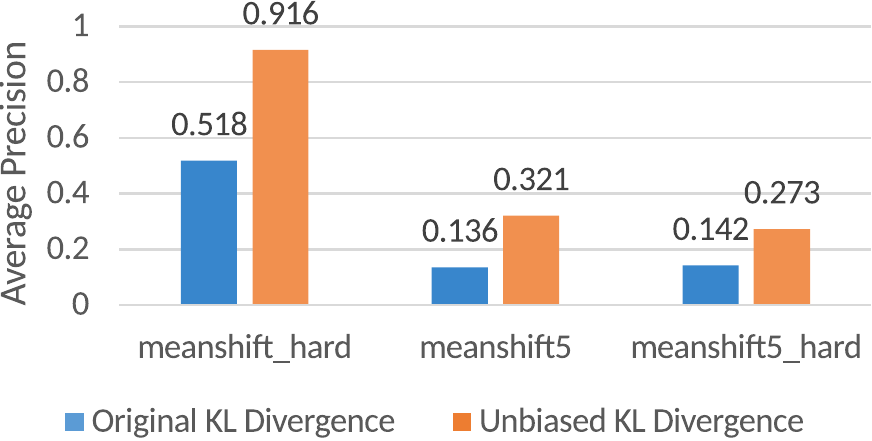}
        \captionof{figure}{Performance of the original and the unbiased KL divergence on test cases with multiple or subtle anomalies.}
        \label{fig:performance-unbiased}
    \end{minipage}
\end{figure*}

Since the detection of anomalous regions in spatio-temporal data is rather a \textit{detection} than a \textit{classification} task, we do not use the \textit{Area under the ROC Curve (AUC)} as performance criterion like many works on point-wise anomaly detection do, but quantify the performance in terms of \textit{Average Precision (AP)} with an Intersection over Union (IoU) criterion that allows an overlap between 50\% and 100\%.

Hotelling's $T^2$ \cite{macgregor1995statistical} and Robust Kernel Density Estimation (RKDE) \cite{kim2012rkde} are used as baselines for the comparison. For RKDE, a Gaussian kernel with a standard deviation of $1.0$ and the Hampel loss function are used. We obtain interval detections from those point-wise baselines by grouping contiguous detections based on multiple thresholds and applying non-maximum suppression afterwards. The overlap threshold for non-maximum suppression is set to 0 in all experiments to obtain non-overlapping intervals only. To be fair, MDI also has to compete with the baselines on the task they have been designed for, i.e., point-wise anomaly detection, by means of AUC. The interval detections can be converted to point-wise detections easily by taking the score of the interval a sample belongs to as score for that sample.

\Cref{fig:performance-comparison} shows that the performance of the MDI algorithm using the Gaussian model is clearly superior on the entire synthetic data set compared to the baselines by means of Mean AP and even on the task of point-wise anomaly detection measured by AUC. The $\divergence{KL}(p_I, p_\Omega)$ polarity of the KL divergence has been used in all experiments following the argumentation in \cref{sec:KL-polarity}. In addition, the performance of the unbiased variant $\divergence{U-KL}(p_I, p_\Omega)$ is reported for the Gaussian model. The parameters of time-delay embedding have been fixed to $\kappa=6, \tau=2$ which we have empirically found to be suitable for this data set. For KDE, we used a Gaussian kernel with bandwidth $1.0$.

While MDI KDE is already superior to the baselines, it is significantly outperformed by MDI Gaussian, which improves on the best baseline by 286\%. This discrepancy between the MDI algorithm using KDE and using Gaussian models is mainly due to time-delay embedding, which is particularly useful for the Gaussian model, because it takes correlations of the variables into account, as opposed to KDE. As can be seen in \cref{fig:td-effect}, the Gaussian model would be worse than KDE and on par with the baselines without time-delay embedding.

Considering the Mean AP on this synthetic data set, the unbiased KL divergence did not perform better than the original KL divergence. However, on the test cases \texttt{meanshift5}, \texttt{meanshift5\_hard}, and \texttt{meanshift\_hard} it achieved an AP twice as high as that of $\divergence{KL}(p_I, p_\Omega)$, which was poor on those data sets (see \cref{fig:performance-unbiased}). Since real data sets are also likely to contain multiple anomalies, we expect $\divergence{U-KL}$ to be a more reliable divergence measure in practice.

Another interesting result is that cross entropy was the best performing divergence measure. This shows the advantage of reducing the impact of the inner distribution $p_I$, which is estimated from very few samples. However, it may perform less reliably on real data whose entropy varies more widely over time than in this synthetic benchmark.

\begin{figure*}
    \begin{subfigure}{0.49\linewidth}%
        \centering
        \caption{Proposal Recall}%
        \includegraphics[width=0.9\linewidth]{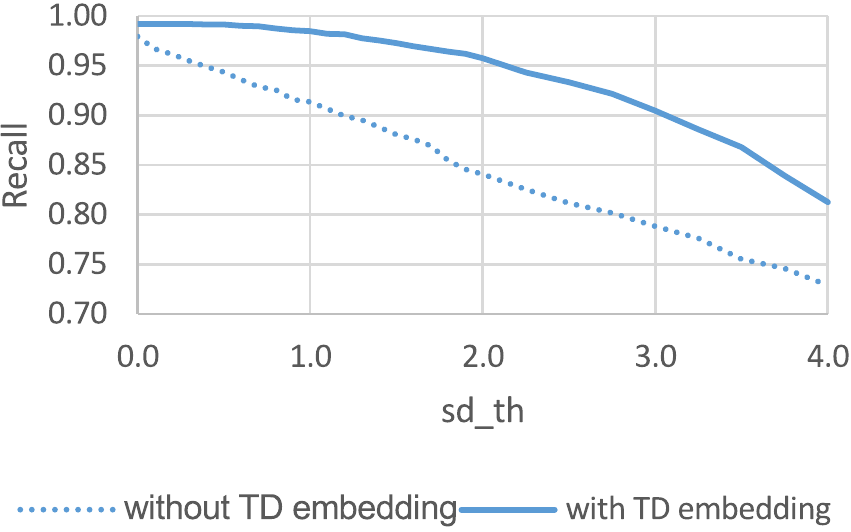}%
        \label{fig:proposal-recall}%
    \end{subfigure}%
    \hfill%
    \begin{subfigure}{0.49\linewidth}%
        \centering
        \caption{Effect of Interval Proposals}%
        \includegraphics[width=0.9\linewidth]{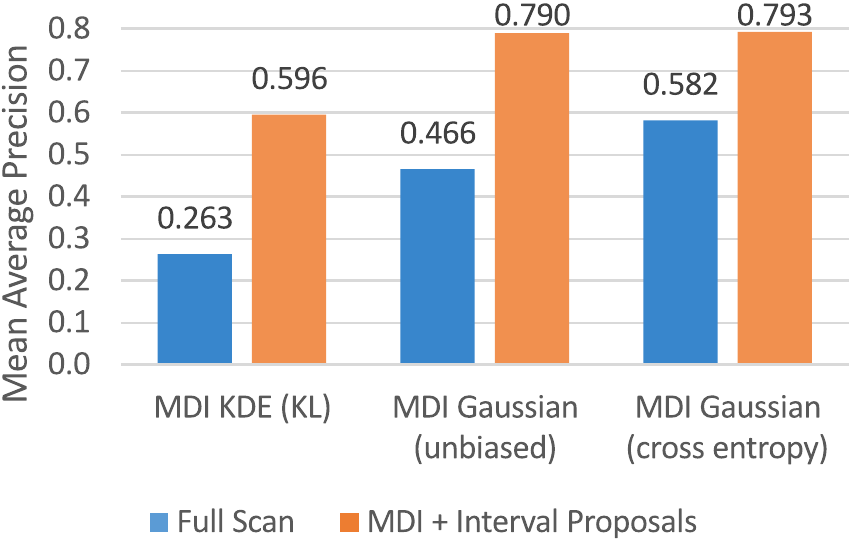}%
        \label{fig:proposal-effect}%
    \end{subfigure}%
    \caption{(\subref{fig:proposal-recall}) Recall of interval proposals without time-delay embedding and with $\kappa=6, \tau=2$ on the synthetic data set for different proposal thresholds. (\subref{fig:proposal-effect}) Effect of interval proposals on the Mean Average Precision of different variants of the MDI algorithm on the synthetic data set.}
    \label{fig:eval-proposals}
\end{figure*}

The Jensen-Shannon divergence performed best for the KDE method, but worst for the Gaussian model. This can be explained by the lack of a closed-form solution for the JS divergence, so that it has to be approximated from the data, while the KL divergence of two Gaussians can be computed exactly. This advantage of the combination of the KL divergence with Gaussians models is, thus, not only beneficial with respect to the run-time of the algorithm, but also with respect to its detection performance.

The differences between the results in \cref{fig:performance-comparison} are significant on a level of 5\% according to the permutation test.

\subsection{Interval Proposals}
\label{sec:eval-proposals}

In order not to sacrifice detection performance for the sake of speed, the interval proposal method described in \cref{sec:proposals} has to act as a high-recall system proposing the majority of anomalous intervals. This can be controlled to some degree by adjusting the threshold $\theta = \mu + \vartheta \cdot \sigma $ applied to the point-wise scores, where $\mu$ and $\sigma$ are the empirical mean and standard deviation of the point-wise scores, respectively. To find a suitable value for the hyper-parameter $\vartheta$, we have evaluated the recall of the proposed intervals for different values of $\vartheta \in [0,4]$ using the usual IoU measure for distinguishing between true and false positive detections. The results in \cref{fig:proposal-recall} show that time-delay embedding is of a great benefit in this scenario too. Based on these results, we selected $\vartheta = 1.5$ for subsequent experiments, which still provides a recall of 97\% and is already able to reduce the number of intervals to be analyzed in detail significantly.

The processing of all the 1100 time-series from the synthetic data set, which took 216 seconds on an Intel Core\texttrademark\ i7-3930K with 3.20GHz and eight virtual cores using the Gaussian model and the unbiased KL divergence after the usual time-delay embedding with $\kappa=6, \tau=2$, could be reduced to 5.2 seconds using interval proposals. This corresponds to a speed-up by more than 40 times.

Though impressive, the speed-up was expected. What was not expected, however, is that the use of interval proposals also increased the detection performance of the entire algorithm by up to 125\%, depending on the divergence. The exact average precision achieved by the algorithm on the synthetic data set with a full scan over all intervals and with interval proposals is shown in \cref{fig:proposal-effect}. This improvement is also reflected by the AUC scores not reported here and is, hence, not specific to the evaluation criterion. A possible explanation for this phenomenon is that some intervals that are uninteresting but distracting for the MDI algorithm are not even proposed for detailed analysis.

\section{Application Examples on Real Data}
\label{sec:applications}

The following application examples on real data from various different domains are intended to complement the quantitative results presented above with a demonstration of the feasibility of our approach for real problems.

\subsection{Detection of North Sea Storms}
\label{sec:exp-storms}

To demonstrate the efficiency of the MDI algorithm on long time-series, we apply it to storm detection in climate data: The coastDat-1 hindcast \cite{coastDat} is a reconstruction of various marine climate variables measured at several locations over the southern North Sea between 51\textdegree\ N, 3\textdegree\ W and 56\textdegree\ N, 10.5\textdegree\ E with an hourly resolution over the 50 years from 1958 to 2007, i.e., approximately 450,000 time steps. Since measurements are not available at locations over land, we select the subset of the data between 53.9\textdegree\ N, 0\textdegree\ E and 56\textdegree\ N, 7.7\textdegree\ E, which results in a regular spatial grid of size $78 \times 43$ located entirely over the sea (cf. \cref{fig:coastdat-map}). Because cyclones and other storms usually have a large spatial extent and move over the region covered by the measurements, we reduce the spatio-temporal data to purely temporal data in this experiment by averaging over all spatial locations. The variables used for this experiment are significant wave height, mean wave period and wind speed.

We apply the MDI algorithm to that data set using the Gaussian model and the unbiased KL divergence. Since North Sea storms lasting longer than 3 days are usually considered two independent storms, the maximum length of the possible intervals is set to 72 hours, while the minimum length is set to 12 hours. The parameters of time-delay embedding are fixed to $\kappa=3, \tau=1$.

28 out of the top 50 and 7 out of the top 10 detections returned by the algorithm can be associated with well-known historic storms. The highest scoring detection is the so-called ``Hamburg-Flut'' which flooded one fifth of Hamburg in February 1962 and caused 340 deaths. Also among the top 5 is the ``North Frisian Flood'', which was a severe surge in November 1981 and lead to several dike breaches in Denmark.

A visual inspection of the remaining 22 detections revealed, that almost all of them are North Sea storms as well. Only 4 of them are not storms, but the opposite: they span times of extremely calm sea conditions with nearly no wind and very low waves, which is some kind of anomaly as well.

A list of the top 50 detections can be found in \cref{app:coastdat-detections} and animated heatmaps of the three variables during the detected time-frames are shown on our web page: \url{http://www.inf-cv.uni-jena.de/libmaxdiv_applications.html}.

Processing this comparatively long time-series using 8 parallel threads took 27 seconds. This time can be reduced further to half a second by using interval proposals without changing the top 10 detections significantly. This supports the assumption, that the novel proposal method does not only perform well on synthetic, but also on real data.

\subsection{Spatio-Temporal Detection of Low Pressure Areas}
\label{sec:exp-slp}

As a genuine spatio-temporal use-case, we have also applied the MDI algorithm to a time-series with daily sea-level pressure (SLP) measurements over the North Atlantic Sea with a much wider spatial coverage than in the previous experiment. For this purpose, we selected a subset of the NCEP/NCAR reanalysis \cite{kalnay1996ncep} covering the years from 1957 to 2011. This results in a time-series of about 20,000 days. The spatial resolution of 2.5\textdegree\ degrees is rather coarse and the locations are organized in a regular grid of size $28 \times 17$ covering the area between 25\textdegree\ N, 52.5\textdegree\ W and 65\textdegree\ N, 15\textdegree\ E.

Again, the MDI algorithm with the Gaussian model and the unbiased KL divergence is applied to this time-series to detect low-pressure fields, which are related to storms. Regarding the time dimension, we apply time-delay embedding with $\kappa=3, \tau=1$ and search for intervals of size between 3 and 10 days. Concerning space, we do not apply any embedding for now and set a minimum size of $7.5\degree \times 7.5\degree$, but no maximum. 7 out of the top 20 detections could be associated with known historic storms.

A visual inspection of the results shows that the MDI algorithm is not only capable of detecting occurrences of anomalous low-pressure fields over time, but also their spatial location. This can be seen in the animations on our web page: \url{http://www.inf-cv.uni-jena.de/libmaxdiv_applications.html}. A few snapshots and a list of detections are also shown in \cref{app:slp-detections}.

It is not necessary to apply spatial-neighbor embedding in this scenario, since we are not interested in spatial outliers, but only in the location of temporal outliers. We have also experimented with applying spatial-neighbor embedding and it led to the detection of some high-pressure fields surrounded by low-pressure fields. Since high-pressure fields are both larger and more common in this time-series, they are not detected as temporal anomalies.

Since we did not set a maximum spatial extent of anomalous regions, the algorithm took 4 hours to process this spatio-temporal time-series. This could, however, be reduced to 22 seconds using our interval proposal technique, with only a minor loss of localization accuracy.

\begin{figure}
    \centering
    \includegraphics[width=0.6\linewidth]{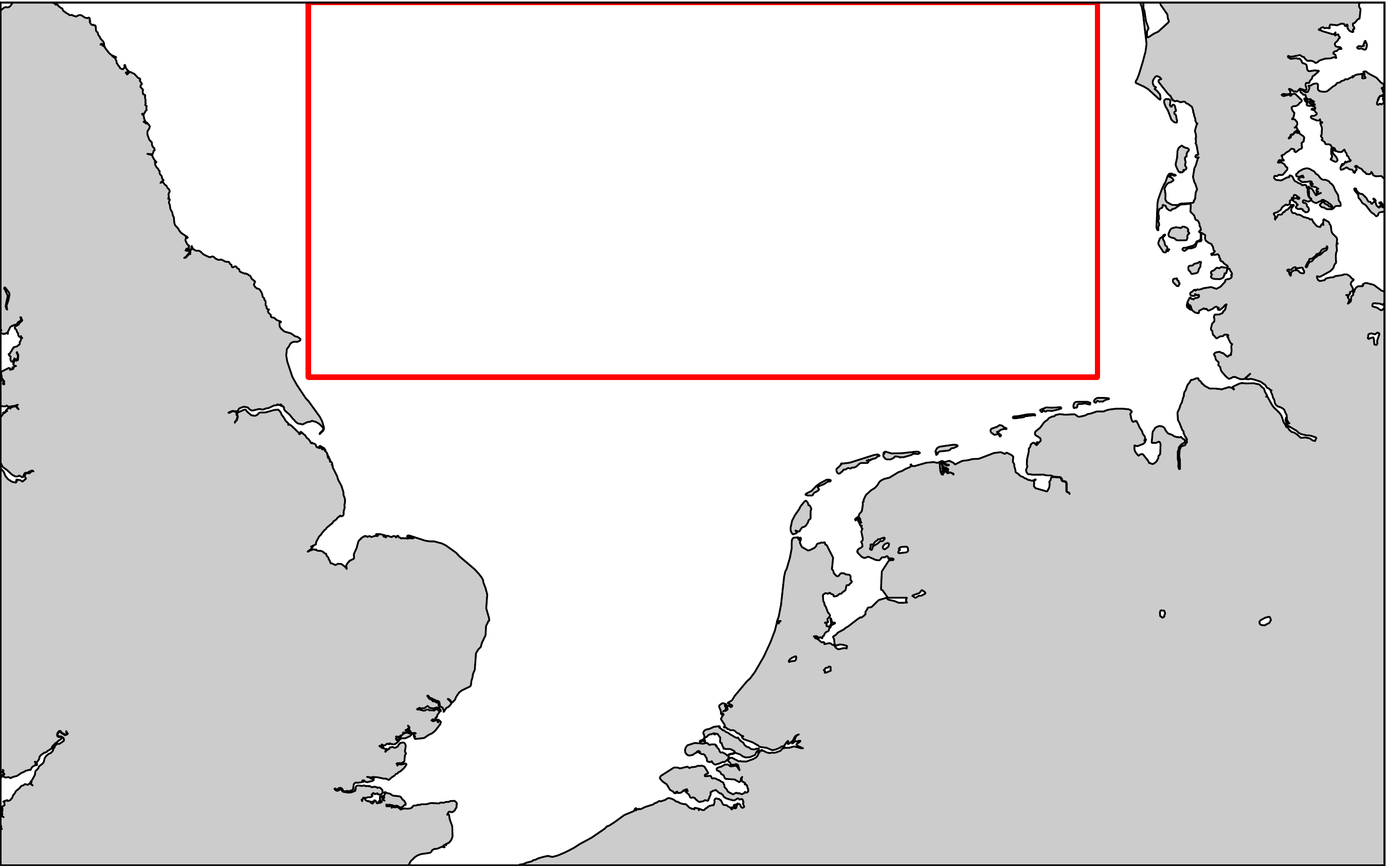}
    \caption{Map of the area covered by the coastDat dataset. The highlighted box denotes the area from which data have been aggregated for our experiment.}
    \label{fig:coastdat-map}
\end{figure}

\subsection{Stylistic Anomalies in Texts of Natural Language}
\label{sec:exp-nlp}

By employing a transformation from the domain of natural language to real-valued features, the MDI algorithm can also be applied to written texts. One important task in Natural Language Processing (NLP) is, for example, the identification of paragraphs written in a different language than the remainder of the document. Such a segmentation can be used as a pre-processing step for the actual, language-specific processing.

In order to simulate such a scenario, we use a subset of the \textit{Europarl} corpus \cite{koehn2005europarl}, which is a sentence-aligned parallel corpus extracted from the proceedings of the European Parliament in 21 different languages. The 33,334 English sentences from the \textit{COMTRANS} subset of \textit{Europarl}, which is bundled with the Natural Language Toolkit (NLTK)  for Python, serve as a basis and 5 random sequences of between 10 and 50 sentences are replaced by their German counterparts to create a semantically coherent mixed-language text.

We employ a simple transformation of sentences to feature vectors: Since the distribution of letter frequencies varies across languages, each sentence is represented by a 27-dimensional vector whose first element is the average word length in the sentence and the remaining 26 components are the absolute frequencies of the letters ``a'' to ``z'' (case-insensitive). German umlauts are ignored since they would make the identification of German sentences too easy.

The MDI algorithm using the unbiased KL divergence is then applied in order to search for anomalous sequences of between 10 and 50 sentences in the mixed-language text after sentence-wise transformation to the feature space. Because the number of features is quite high in relation to the number of samples in an interval, we use a global covariance matrix shared among the Gaussian models and do not apply time-delay embedding.

The top 5 detections returned by the algorithm correspond to the 5 German paragraphs that have been injected into the English text. The localization is quite accurate, though not perfect: on average, the boundaries of the detected paragraphs are off by 1.4 sentences from the ground-truth. The next 5 detections are mainly tables and enumerations, which are also an anomaly compared with the usual dialog style of the parliament proceedings.

For this scenario, we had designed the features specifically for the task of language identification. To see what else would be possible with a smaller bias towards a specific application, we have also applied the algorithm to the \nth{1} Book of Moses (Genesis) in the King James Version of the bible, where we use \texttt{word2vec} \cite{mikolov2013efficient} for word-wise feature embeddings. \texttt{word2vec} learns real-valued vector representations of words in a way, so that the representations of words that occur more often in similar contexts have a smaller Euclidean distance. The embeddings used for this experiment have been learned from the Brown corpus using the continuous skip-gram model and we have chosen a dimensionality of 50 for the vector space, which is rather low for \texttt{word2vec} models, but still tractable for the Gaussian probability density model. Words which have not been seen by the model during training are treated as missing values.

The top 10 detections of sequences of between 50 and 500 words according to the unbiased KL divergence are provided in \cref{app:genesis-detections}. The first five of those are, without exception, genealogies, which can indeed be considered as anomalies, because they are long lists of names of fathers, sons and wives, connected by repeating phrases. The \nth{6} detection is a dialog between God and Abraham, where Abraham bargains with God and tries to convince him not to destroy the town Sodom. This episode is another example for stylistic anomalies, since the dialog is a concatenation of very similar question-answer pairs with only slight modifications.

Due to the rather wide limits on the possible size of anomalous intervals, the analysis of the entire book Genesis, a sequence of 44,764 words, took a total of 9 minutes, where we have not yet used interval proposals.

\subsection{Anomalies in Videos}
\label{sec:exp-video}

The detection of unusual events in videos is another important task, e.g., in the domain of video surveillance or industrial control systems. Though videos are already represented as multivariate spatio-temporal time-series with usually 3 variables (RGB channels), a semantically more meaningful representation can be obtained by extracting features from a Convolutional Neural Network (CNN).

In this experiment, we use a video of a traffic scene from the ViSOR repository \cite{vezzani2010visor}. It has a length of 60 seconds (1495 frames) and a rather low resolution of $360 \times 288$ pixels. The video shows a street and a side-walk with a varying frequency of cars crossing the captured area horizontally in both directions. At one point, a group of two pedestrians and one cyclist appears on the side-walk and crosses the area from right to left at a low speed. Another sequence at the end of the video shows a single cyclist riding along the side-walk in the opposite direction at a higher speed. Altogether, 26 seconds of the video contain moving objects and 34 seconds just show an empty street. The nominal state of the scene hence is not unambiguous.

We extract features for each frame of the video from the \texttt{conv5} layer of CaffeNet \cite{jia2014caffe}, which reduces the spatial resolution to $22 \times 17$, but increases the number of feature dimensions to 256. This rather large feature space is then reduced to 16 dimensions using PCA and the MDI algorithm is applied to search for anomalous sub-blocks with a minimum spatial extent of $10 \times 5$ cells and a length between 3 and 12 seconds. The time-delay embedding parameters are fixed to $\kappa=3, \tau=4$ for capturing half a second as context without increasing the number of dimensions too much. We apply the MDI algorithm with both the unbiased KL divergence and cross entropy as divergence measures. The Gaussian distribution model is employed in both cases.

The results (some snapshots are shown in \cref{fig:visor-detections}) exhibit an interesting difference between the two divergence measures: The KL divergence detects a sub-sequence of approximately 10 seconds where absolutely no objects cross the captured area. Thus, car traffic is identified as normal behavior and long spans of time without any traffic are considered as anomalous, because they have a very low entropy and the KL divergence penalizes the entropy of all other intervals, as opposed to cross entropy which does not take the entropy of the detected interval into account. Another detection occurs when the group of pedestrians enters the area. The localization, however, is rather fuzzy and spans nearly the entire frame. Cross entropy, on the other hand, seems to identify the state of low or no traffic as normal behavior and yields two detections at the beginning and the end of the video where the frequency of cars is higher than in the rest of the video. It detects the pedestrians too, but with a better localization accuracy. This detection, however, does not cover the entire side-walk, since the pedestrians are moving from right to left and the algorithm is not designed for tracking moving anomalies.

Without using interval proposals, the comparatively high number of features combined with the large spatial search space would result in a processing time of 13 hours for this video.
This can be reduced to 5 minutes using our novel interval proposal technique.

\begin{figure}
    \begin{subfigure}{0.23\linewidth}%
        \includegraphics[width=\linewidth]{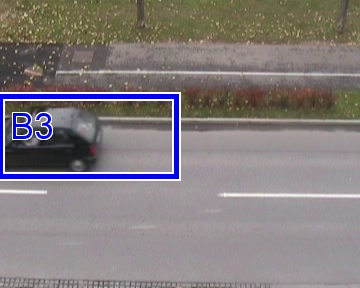}%
        \caption{3 s}%
    \end{subfigure}%
    \hfill%
    \begin{subfigure}{0.23\linewidth}%
        \includegraphics[width=\linewidth]{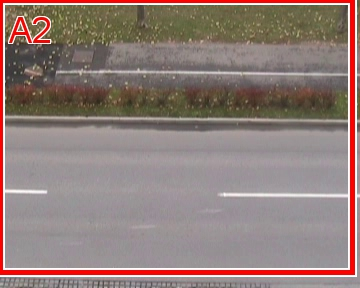}%
        \caption{13 s}%
    \end{subfigure}%
    \hfill%
    \begin{subfigure}{0.23\linewidth}%
        \includegraphics[width=\linewidth]{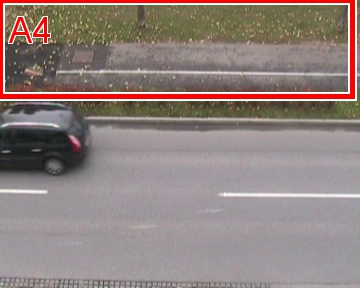}%
        \caption{23 s}%
    \end{subfigure}%
    \hfill%
    \begin{subfigure}{0.23\linewidth}%
        \includegraphics[width=\linewidth]{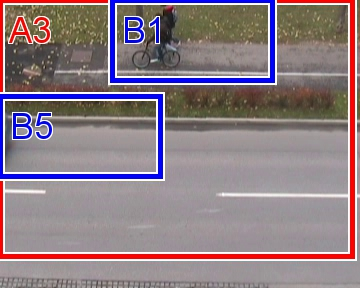}%
        \caption{30 s}%
    \end{subfigure}%
    \caption{Snapshots from the example video with corresponding detections. Regions detected using the unbiased KL divergence start with the character ``A'', those detected by cross entropy start with ``B''. The full video can be found on our web page: \protect\url{http://www.inf-cv.uni-jena.de/libmaxdiv_applications.html}.}
    \label{fig:visor-detections}
\end{figure}

\section{Summary and Conclusions}
\label{sec:conclusions}

We have introduced a novel unsupervised algorithm for anomaly detection that is suitable for analyzing large multivariate time-series and can detect anomalous {\em regions} not only in temporal but also in spatio-temporal data from various domains. The proposed MDI algorithm outperforms existing anomaly detection techniques, while being comparatively time efficient, thanks to an efficient implementation and a novel interval proposal technique that excludes uninteresting parts of the data from in-depth analysis. Moreover, we have exposed a bias of the Kullback-Leibler (KL) divergence towards smaller intervals and proposed an unbiased KL divergence that is superior when applied to real data. We have also investigated other divergence measures and found that the use of cross entropy can result in improved performance for data with a low variability of entropy.

Various experiments on data from different domains, including climate analysis, natural language processing and video surveillance, have shown that the algorithm proposed in this work can serve as a generic, unsupervised anomaly detection technique that can facilitate tasks such as process control, data analysis and knowledge discovery.
These application examples emphasize the importance of interval-based anomaly detection techniques, and we hope that our work is able to motivate further research in this area.

For processing data with a large spatial extent or a high number of dimensions, a full scan over all possible sub-blocks of the data would be prohibitively time-consuming. To this end, we have introduced a novel interval proposal technique that can reduce computation time significantly. However, interval proposals usually lead to less accurate detections, which is particularly noticeable with regard to spatial dimensions. Future work might hence investigate applying in-depth analysis not only to the proposed intervals themselves, but also to their neighborhood. An alternative might be a hierarchical approach of successive refinement.

Other open problems to be addressed in the future include efficient probability density estimation in the face of high-dimensional data, the automatic determination of suitable parameters for time-delay embedding, and tracking anomalies moving in space over time. Furthermore, it is often necessary to convince the expert analyst that a detected anomaly really is an anomaly. Thus, future work will include the development of an attribution scheme that can explain which variables or combinations of variables caused a detection and why.

\section*{Acknowledgements}
The support of the project EU H2020-EO-2014 project BACI
``Detecting changes in essential ecosystem and biodiversity properties-towards a
Biosphere Atmosphere Change Index'',
contract 640176, is gratefully acknowledged.

%\IEEEtriggeratref{30}
\bibliographystyle{IEEEtran}
\bibliography{references}

\vspace*{-.5\baselineskip}
\begin{IEEEbiography}[{\includegraphics[width=1in,height=1.25in,clip,keepaspectratio]{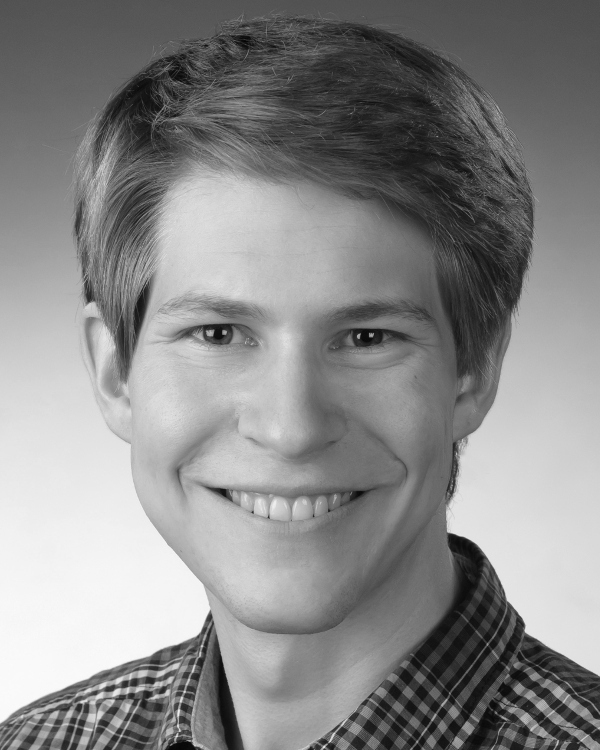}}]{Bj{\"o}rn Barz}
    received the B.Sc. and M.Sc. degrees in computer science with honours from Friedrich Schiller University Jena, Germany, in 2014 and 2016, respectively.
    He is currently working towards the PhD degree at the Computer Vision Group of Joachim Denzler at the University of Jena.
    His research interests are in the field of machine learning, visual object detection, content-based image retrieval, and natural language processing.
\end{IEEEbiography}

\begin{IEEEbiography}[{\includegraphics[width=1in,height=1.25in,clip,keepaspectratio]{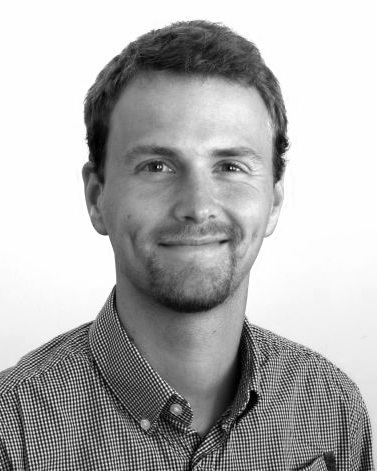}}]{Erik Rodner}
    earned the Diploma degree in Computer Science with honours in 2007 from the Friedrich Schiller University Jena, Germany.
    He received his PhD in 2011 with summa cum laude for his work on learning with few examples, which was done under supervision of Joachim Denzler at the computer vision group of the University of Jena.
    From 2012 to 2013, Erik joined UC Berkeley as a postdoctoral researcher.
    He was senior researcher and lecturer in the computer vision group at the University of Jena from 2013 to 2016 and is now researcher at Carl Zeiss AG.
    His research interests include domain adaptation, deep learning, visual object discovery, active and continuous learning, and scene understanding.
\end{IEEEbiography}

\begin{IEEEbiography}[{\includegraphics[width=1in,height=1.25in,clip,keepaspectratio]{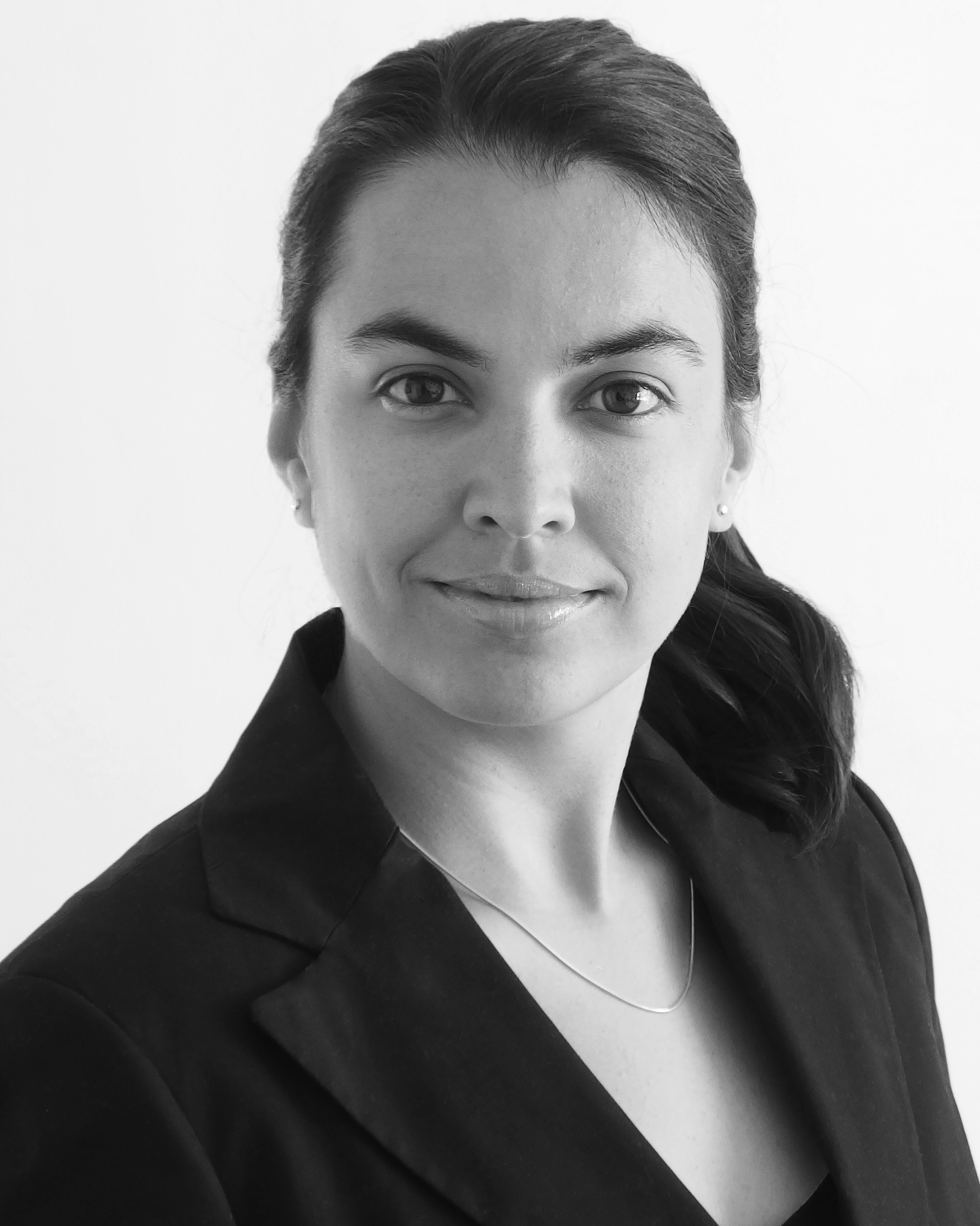}}]{Yanira Guanche Garcia}
    earned the M.Sc. in Coastal and Ports Engineering in 2010 and received her PhD in 2013 from the Universidad de Cantabria, Spain.
    From 2014 to 2015, Yanira joined IFREMER and BRGM in France as a postdoctoral researcher.
    Since 2015, she is a postdoctoral researcher at the computer vision group of Joachim Denzler at the Friedrich Schiller University, Jena, and research coordinator of the Michael Stifel Center for Data-Driven and Simulation Science, Jena.
\end{IEEEbiography}

\begin{IEEEbiography}[{\includegraphics[width=1in,height=1.25in,clip,keepaspectratio]{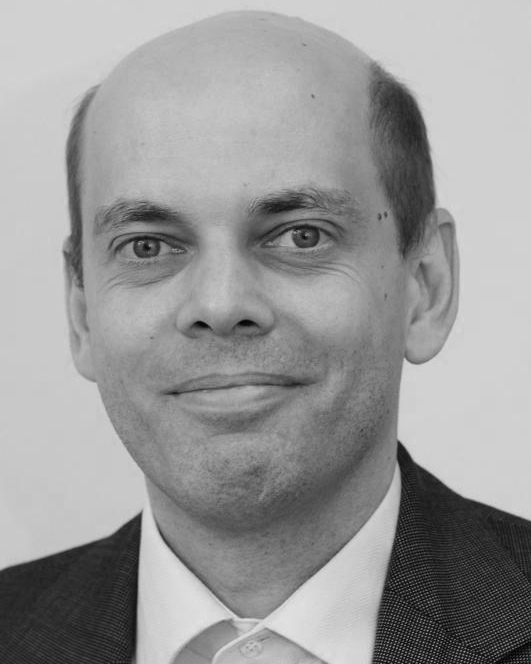}}]{Joachim Denzler}
    earned the degrees ``Diplom-Informatiker'', ``Dr.-Ing.''  and ``Habilitation'' from the University of Erlangen, Germany, in years 1992, 1997, and 2003, respectively.
    Currently, he holds a position as full professor for computer science and is head of the
    Computer Vision Group at the Friedrich Schiller University Jena, Germany.
    He is also Director of the Michael Stifel Center for Data-Driven and Simulation Science, Jena.
    His research interests comprise the automatic analysis, fusion, and understanding of sensor data, especially development of methods for visual recognition tasks and dynamic scene analysis.
    He contributed in the area of active vision, 3D reconstruction, as well as object recognition and tracking.
    He is author and co-author of over 300 journal and conference papers as well as technical articles.
    He is a member of IEEE, IEEE computer society, DAGM, and GI.
\end{IEEEbiography}

%auto-ignore

\newtheorem{theorem}{Theorem}

\newcommand{\fpar}[1]{\noindent\fbox{\parbox{\textwidth}{\scriptsize #1}}}

\newcommand{\hrrule}{\textcolor{red}{\vspace{0.4em}{\hrule height 1pt}\vspace{0.4em}}}
\newcommand{\hbrule}{\textcolor{blue}{\vspace{0.4em}{\hrule height 1pt}\vspace{0.4em}}}
\newcommand{\hbrrule}{\vspace{0.4em}\textcolor{blue}{\hrule height 1pt}\textcolor{red}{\hrule height 1pt}\vspace{0.4em}}

\definecolor{grey}{gray}{0.5}

\onecolumn

\begin{appendices}
\crefalias{section}{appsec}

\section{Generalized extraction of the sum over certain range from a higher-order tensor of cumulative sums}
\label{app:cumsum-extraction}

\vspace{1cm}

\begin{theorem}
    \label{theorem:multi-cumsum}
    
    Let $\mathfrak{X}, C \in \mathbb{R}^{N_1 \times N_2 \times \cdots \times N_M}$ with $\mathfrak{X}$ being an $M^\text{th}$-order data tensor and $C$ being a tensor of cumulative sums given by
    \begin{equation}
    C(k_1, k_2, \dots, k_M) \coloneqq
    \sum_{i_1 = 1}^{k_1} \sum_{i_2 = 1}^{k_2} \cdots \sum_{i_M = 1}^{k_M} \mathfrak{X}(i_1, i_2, \dots, i_M) .
    \end{equation}
    
    The sum over all elements of $\mathfrak{X}$ in the range $\left( j_0^{(1)}, j_1^{(1)} \right] \times \left( j_0^{(2)}, j_1^{(2)} \right] \times \cdots \times \left( j_0^{(M)}, j_1^{(M)} \right]$ can then be reconstructed from $C$ with $2^M - 1$ additions/subtractions according to
    \begin{equation}
    \begin{split}
    & \sum_{i_1 = j_0^{(1)} + 1}^{j_1^{(1)}}
    \sum_{i_2 = j_0^{(2)} + 1}^{j_1^{(2)}} \cdots
    \sum_{i_M = j_0^{(M)} + 1}^{j_1^{(M)}} \mathfrak{X}(i_1, i_2, \dots, i_M) \\
    = & \sum_{(i_1, i_2, \dots, i_M) \in \{0,1\}^M} (-1)^{M - \left( \sum_{m=1}^{M} i_m \right)}
    \cdot C \left( j_{i_1}^{(1)}, j_{i_2}^{(2)}, \dots, j_{i_M}^{(M)} \right) .
    \end{split}
    \label{eq:multi-cumsum-recons}
    \end{equation}
    %\hfill $\blacksquare$
\end{theorem}

\begin{proof}
    For the basic case of $M = 1$ it can easily be seen that
    \[
    \sum_{i=j_0+1}^{j_1} \mathfrak{X}(i)
    = \sum_{i=1}^{j_1} \mathfrak{X}(i) - \sum_{i=1}^{j_0} \mathfrak{X}(i)
    = C(j_1) - C(j_0)
    = \sum_{i \in \{0,1\}} (-1)^{1-i} \cdot C(j_i) .
    \]
    
    Now assume that theorem \ref{theorem:multi-cumsum} holds for $1 \le M' < M$. Applying it for $M-1$ gives
    \[
    \begin{split}
    & \sum_{i_1 = j_0^{(1)} + 1}^{j_1^{(1)}}
    \sum_{i_2 = j_0^{(2)} + 1}^{j_1^{(2)}} \cdots
    \sum_{i_M = j_0^{(M)} + 1}^{j_1^{(M)}} \mathfrak{X}(i_1, i_2, \dots, i_M) \\
    = & \sum_{i_M = j_0^{(M)}+1}^{j_1^{(M)}} \left(
    \sum_{(i_1, i_2, \dots, i_{M-1}) \in \{0,1\}^{M-1}}
    (-1)^{M - 1 - \left( \sum_{m=1}^{M-1} i_m \right)}
    \cdot C \left(
    j_{i_1}^{(1)}, j_{i_2}^{(2)}, \dots, j_{i_{M-1}}^{(M-1)}, i_M
    \right)
    \right) \\
    = & \sum_{(i_1, i_2, \dots, i_{M-1}) \in \{0,1\}^{M-1}} \left(
    (-1)^{M - 1 - \left( \sum_{m=1}^{M-1} i_m \right)}
    \cdot \sum_{i_M = j_0^{(M)}+1}^{j_1^{(M)}} C \left(
    j_{i_1}^{(1)}, j_{i_2}^{(2)}, \dots, j_{i_{M-1}}^{(M-1)}, i_M
    \right)
    \right) .
    \end{split}
    \]
    
    Since the first $M-1$ indices of $C \left( j_{i_1}^{(1)}, j_{i_2}^{(2)}, \dots, j_{i_{M-1}}^{(M-1)}, i_M \right)$ are fixed in the scope of the inner sum and only the last index varies, the basic case for $M=1$ can be applied to that inner sum expression, transforming the right-hand side of the equation to
    \[
    \begin{split}
    & \sum_{(i_1, \dots, i_{M-1}) \in \{0,1\}^{M-1}} \left(
    (-1)^{M - 1 - \left( \sum_{m=1}^{M-1} i_m \right)}
    \cdot \sum_{i_M \in \{0,1\}} (-1)^{1-i_M} \cdot C \left(
    j_{i_1}^{(1)}, \dots, j_{i_M}^{(M)}
    \right)
    \right) \\
    = & \sum_{(i_1, i_2, \dots, i_M) \in \{0,1\}^M} (-1)^{M - \left( \sum_{m=1}^{M} i_m \right)}
    \cdot C \left( j_{i_1}^{(1)}, j_{i_2}^{(2)}, \dots, j_{i_M}^{(M)} \right) .
    \end{split}
    \]
\end{proof}

\clearpage

\section{North Sea Storm Detections}
\label{app:coastdat-detections}

\iftrue
Each heatmap shows the state of the three variables at the middle of the top 5 detected time frames. The static red box marks the spatial subset of the data that has been used for this experiment described in \cref{sec:exp-storms}.
Heatmaps are best viewed in color.

Animated heatmaps for more detections can be found on our web page: \url{http://www.inf-cv.uni-jena.de/libmaxdiv_applications.html}. \\

\noindent\includegraphics[width=\linewidth]{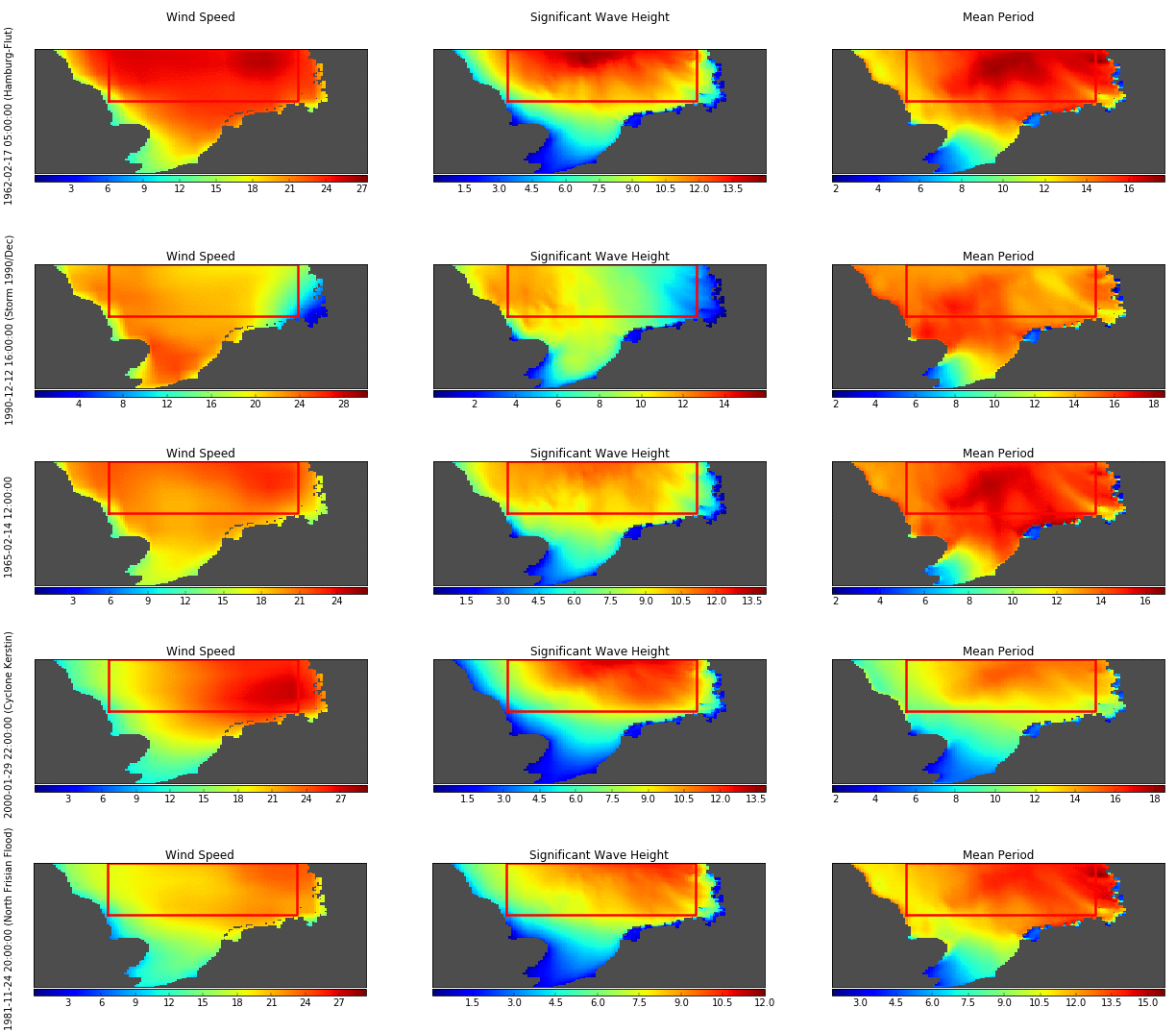} \\
\clearpage

{
    \noindent
    \fontsize{10pt}{11pt}\selectfont
    \begin{tabular}{c|lll}
        \toprule 
        \# & Timeframe & Score & Historical Storm \\ 
        \midrule 
        1 & 1962-02-16 05:00 -- 1962-02-18 06:00 & 831.635 & Hamburg-Flut (Feb 16-17) \\
        2 & 1990-12-11 22:00 -- 1990-12-13 10:00 & 824.840 & Storm 1990/Dec (Dec 12) \\
        3 & 1965-02-13 01:00 -- 1965-02-15 23:00 & 797.781 &  \\
        4 & 2000-01-29 12:00 -- 2000-01-31 02:00 & 796.528 & Cyclone Kerstin (Jan 29-31) \\
        5 & 1981-11-23 19:00 -- 1981-11-25 22:00 & 745.951 & North Frisian Flood (Nov 24) \\
        6 & 1989-02-13 11:00 -- 1989-02-16 10:00 & 714.684 & Storm 1989/Feb (Feb 14) \\
        7 & 1988-02-28 04:00 -- 1988-03-02 03:00 & 710.495 &  \\
        8 & 1973-12-12 16:00 -- 1973-12-15 15:00 & 673.886 & Storm 1973/Dec (2) (Dec 13-15) \\
        9 & 1998-12-26 00:00 -- 1998-12-28 05:00 & 658.592 & Cyclone Stephen (Dec 26-27) \\
        10 & 1984-01-02 15:00 -- 1984-01-05 14:00 & 658.306 &  \\
        11 & 1977-11-13 00:00 -- 1977-11-15 23:00 & 592.562 &  \\
        12 & 1980-02-26 00:00 -- 1980-02-28 23:00 & 573.913 &  \\
        13 & 1999-02-04 15:00 -- 1999-02-07 14:00 & 572.603 & Storm 1999/Feb (Feb 05) \\
        14 & 2006-10-31 07:00 -- 2006-11-01 21:00 & 560.806 & Cyclone Britta (Oct 31 - Nov 01) \\
        15 & 1995-01-09 21:00 -- 1995-01-12 20:00 & 554.777 &  \\
        16 & 1983-01-17 20:00 -- 1983-01-20 14:00 & 545.856 & Storm 1983/Jan (Jan 17-20) \\
        17 & 1991-10-17 00:00 -- 1991-10-19 23:00 & 537.879 &  \\
        18 & 1996-11-05 07:00 -- 1996-11-07 05:00 & 519.837 & Storm 1996/Nov (Nov 05-07) \\
        19 & 1976-01-20 10:00 -- 1976-01-23 02:00 & 508.532 & Storm 1976/Jan (2) (Jan 21) \\
        20 & 1993-01-24 02:00 -- 1993-01-27 01:00 & 506.250 & Storm 1993/Jan (2) (Jan 22-25) \\
        21 & 1973-11-19 05:00 -- 1973-11-20 16:00 & 494.595 & Storm 1973/Nov (3) (Nov 19-20) \\
        22 & 1992-12-23 15:00 -- 1992-12-26 14:00 & 491.438 &  \\
        23 & 1977-12-29 22:00 -- 1977-12-31 17:00 & 489.287 &  \\
        24 & 2004-02-07 07:00 -- 2004-02-09 23:00 & 485.346 & Cyclone Ursula (Feb 07-08) \\
        25 & 1984-01-12 01:00 -- 1984-01-15 00:00 & 485.231 & Storm 1984/Jan (Jan 14) \\
        26 & 1991-01-04 19:00 -- 1991-01-07 18:00 & 471.642 & Storm Undine (Jan 02-09) \\
        27 & 1973-11-11 03:00 -- 1973-11-14 02:00 & 471.398 & Storm 1973/Nov (1) (Nov 13-14) \\
        28 & 2004-11-17 09:00 -- 2004-11-20 08:00 & 460.155 &  \\
        29 & 1973-12-04 09:00 -- 1973-12-07 08:00 & 445.639 & Storm 1973/Dec (1) (Dec 06-07) \\
        30 & 1961-03-26 13:00 -- 1961-03-29 01:00 & 423.340 &  \\
        31 & 1994-01-28 06:00 -- 1994-01-31 05:00 & 420.520 & Cyclone Lore (Jan 27-28) \\
        32 & 1980-04-19 03:00 -- 1980-04-21 18:00 & 420.186 &  \\
        33 & 1999-12-23 18:00 -- 1999-12-26 05:00 & 418.482 & Cyclone Lothar (Dec 25-26) \\
        34 & 1988-10-02 07:00 -- 1988-10-05 02:00 & 412.905 &  \\
        35 & 1970-10-19 02:00 -- 1970-10-22 01:00 & 411.187 &  \\
        36 & 2007-01-10 20:00 -- 2007-01-13 18:00 & 407.993 & Cyclone Franz (Jan 11) \\
        37 & 2007-11-07 11:00 -- 2007-11-10 10:00 & 402.426 & Cyclone Tilo (Nov 06-11) \\
        38 & 1990-09-19 07:00 -- 1990-09-22 06:00 & 397.632 &  \\
        39 & 1993-02-19 00:00 -- 1993-02-21 23:00 & 387.598 & Storm 1993/Feb (Feb 20-21) \\
        40 & 1998-10-24 11:00 -- 1998-10-27 08:00 & 382.914 & Cyclone Xylia (Oct 27-28) \\
        41 & 2003-12-12 15:00 -- 2003-12-15 14:00 & 377.435 & Cyclone Fritz (Dec 13-15) \\
        42 & 1991-12-23 06:00 -- 1991-12-26 03:00 & 374.911 &  \\
        43 & 2002-02-19 17:00 -- 2002-02-22 16:00 & 374.026 & Storm 2002/Feb (1) (Feb 21-23) \\
        44 & 1997-03-08 12:00 -- 1997-03-11 11:00 & 371.758 &  \\
        45 & 1959-02-17 02:00 -- 1959-02-20 01:00 & 369.272 &  \\
        46 & 1974-12-11 06:00 -- 1974-12-14 05:00 & 365.459 &  \\
        47 & 1994-12-06 02:00 -- 1994-12-09 01:00 & 358.443 &  \\
        48 & 2001-12-20 08:00 -- 2001-12-23 07:00 & 357.280 &  \\
        49 & 1992-11-18 14:00 -- 1992-11-21 02:00 & 343.076 &  \\
        50 & 2006-12-29 22:00 -- 2007-01-01 21:00 & 340.920 & Cyclone Karla (Dec 30-31) \\
        \bottomrule 
    \end{tabular}
}

\clearpage

\section{Low Pressure Area Detections}
\label{app:slp-detections}

Each row shows the heatmap of Sea Level Pressure at the beginning, the middle, and the end of the detection. The red box marks the detected area. Details on this experiment can be found in \cref{sec:exp-slp}.

Heatmaps are best viewed in color. An animated video showing all detections can be found on our web-page: \url{http://www.inf-cv.uni-jena.de/libmaxdiv_applications.html}.

{
    \noindent
    \centering
    \def\arraystretch{2}
    \begin{tabular}{ @{}>{\centering\scriptsize}m{0.04\linewidth} @{} >{\centering}m{0.28\linewidth} @{\hspace{0.03\linewidth}} >{\centering}m{0.28\linewidth} @{\hspace{0.03\linewidth}} >{\centering\arraybackslash}m{0.28\linewidth} @{\hspace{0.03\linewidth}} }
        & \textbf{Start} & \textbf{Middle} & \textbf{End} \\
        \rotatebox{90}{1996-01-06 -- 1996-01-15} & \includegraphics[width=\linewidth]{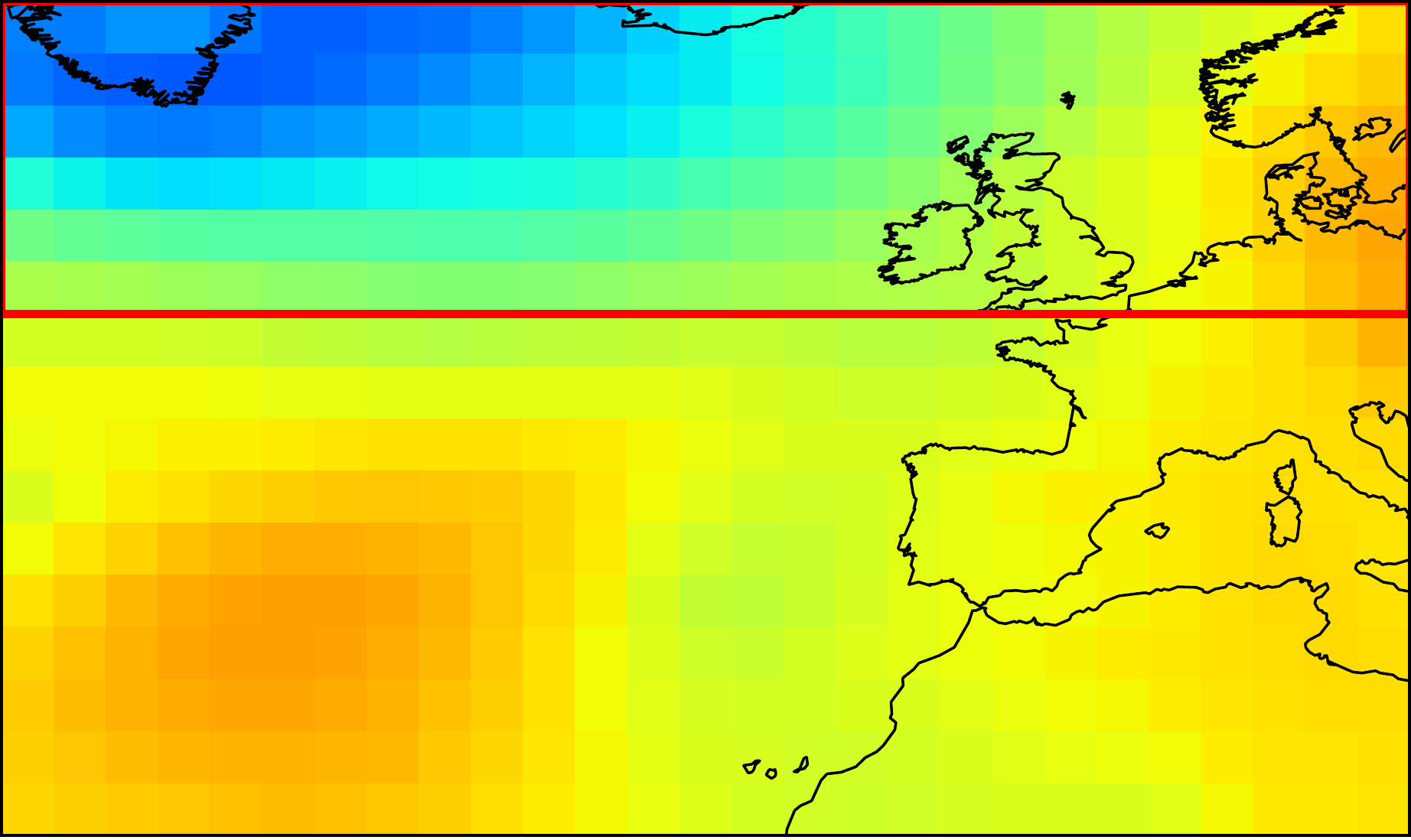} & \includegraphics[width=\linewidth]{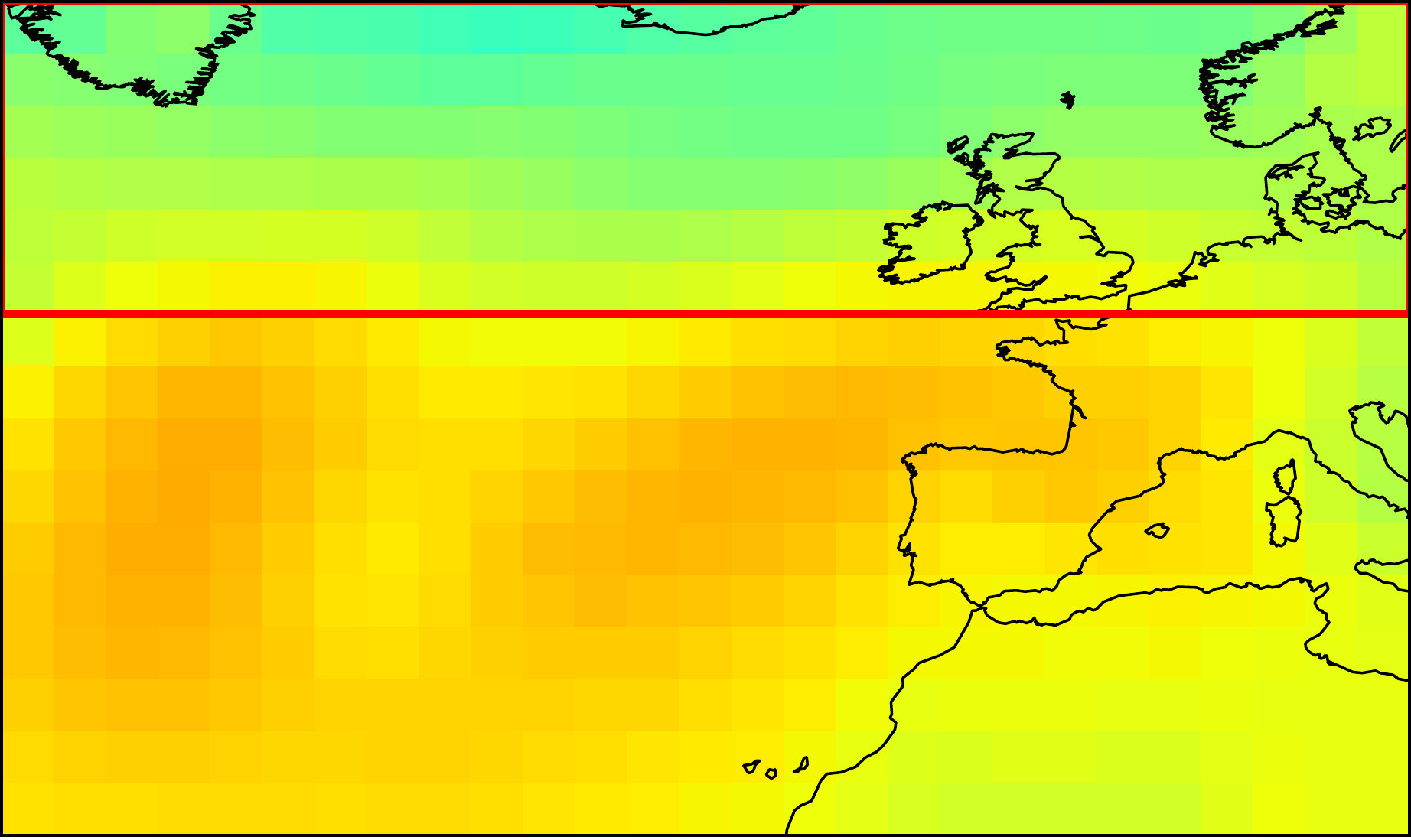} & \includegraphics[width=\linewidth]{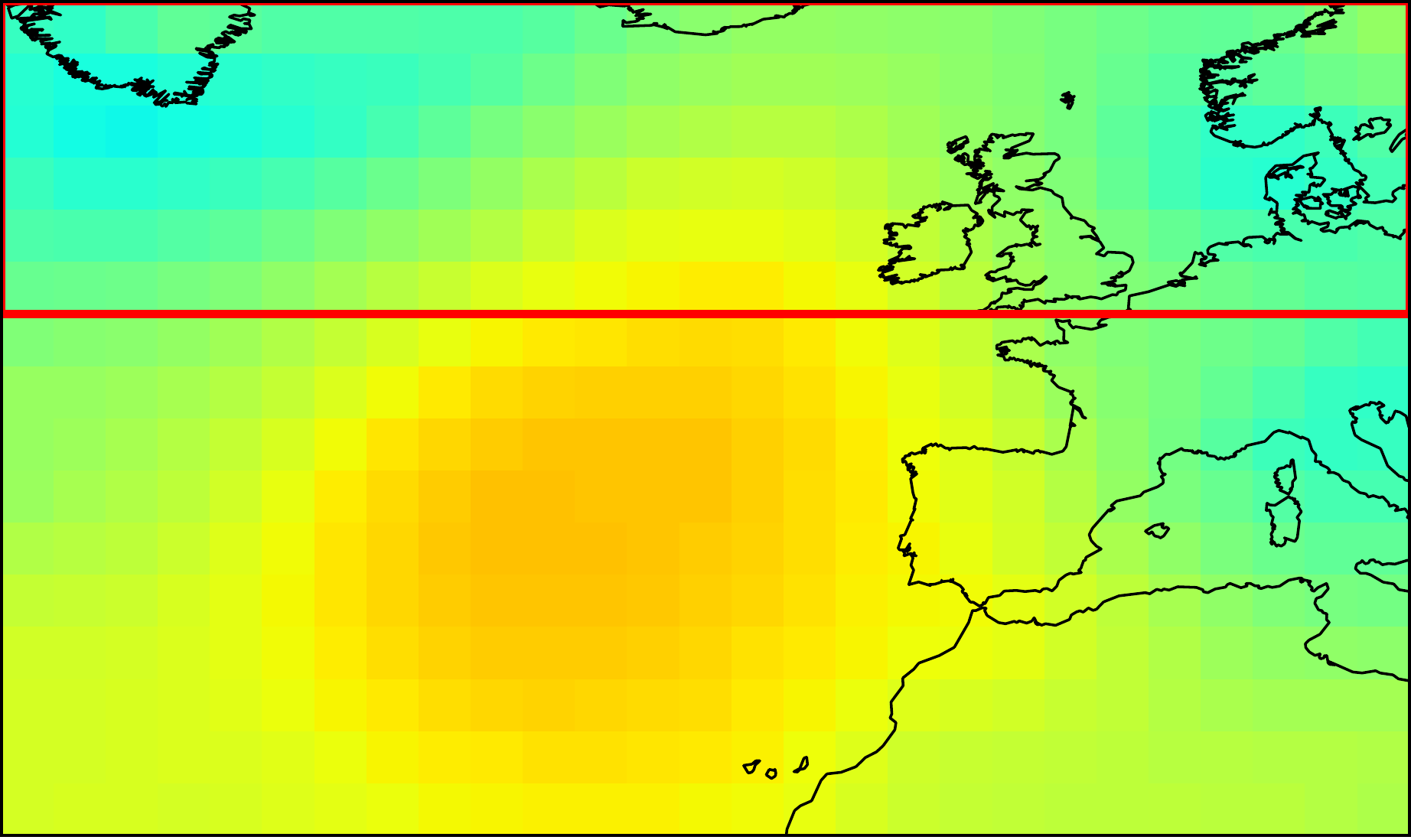} \\
        \rotatebox{90}{1990-01-28 -- 1990-02-06} & \includegraphics[width=\linewidth]{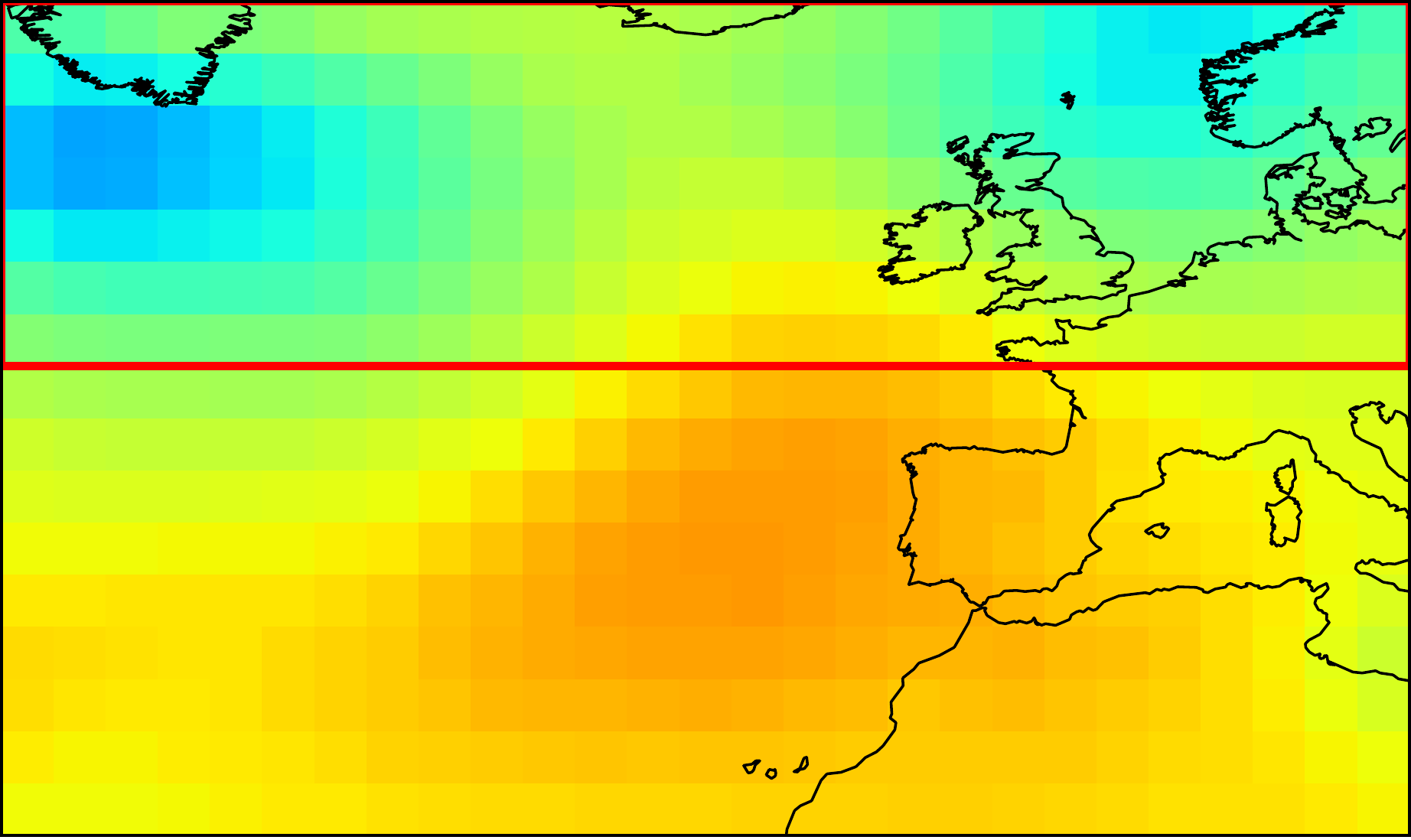} & \includegraphics[width=\linewidth]{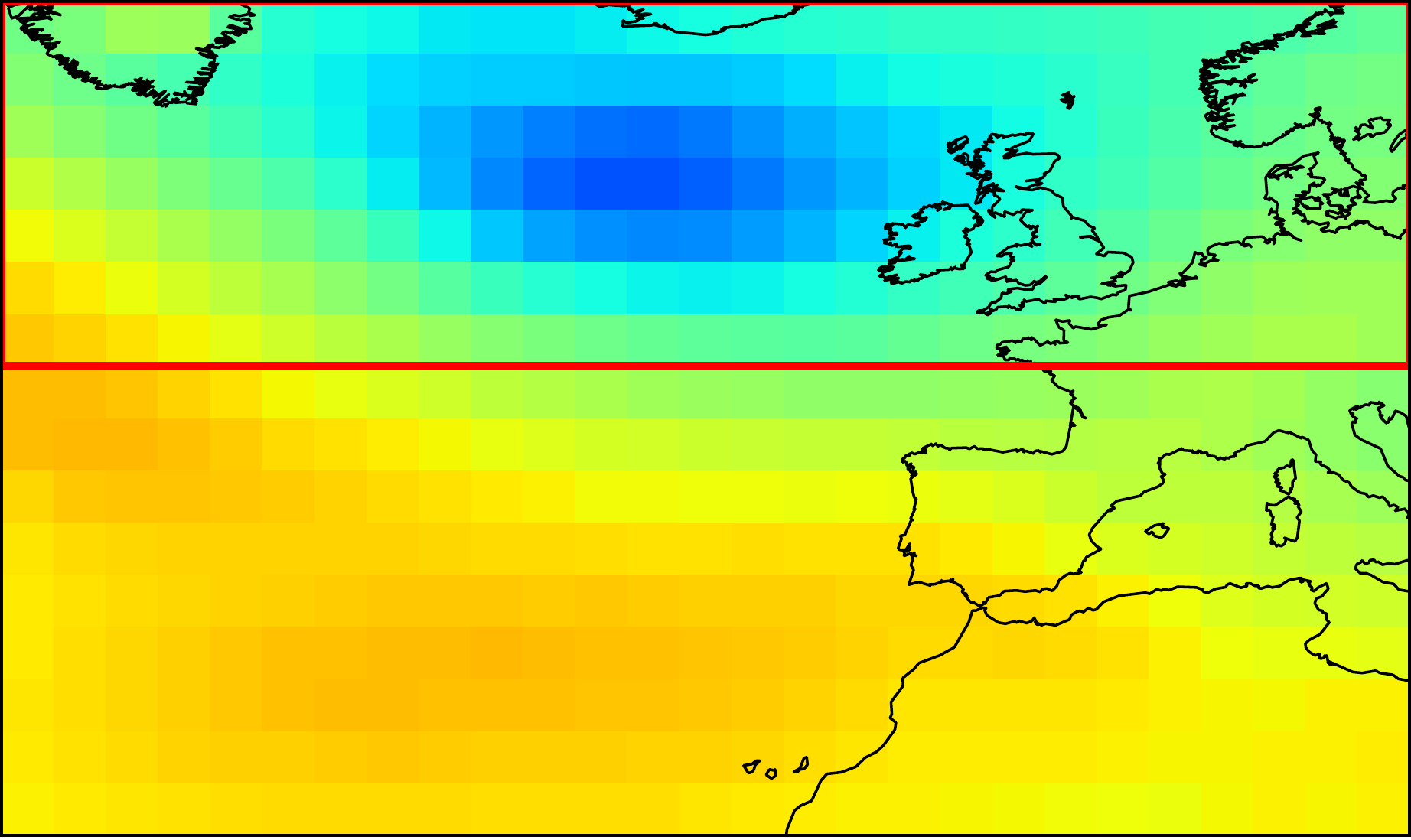} & \includegraphics[width=\linewidth]{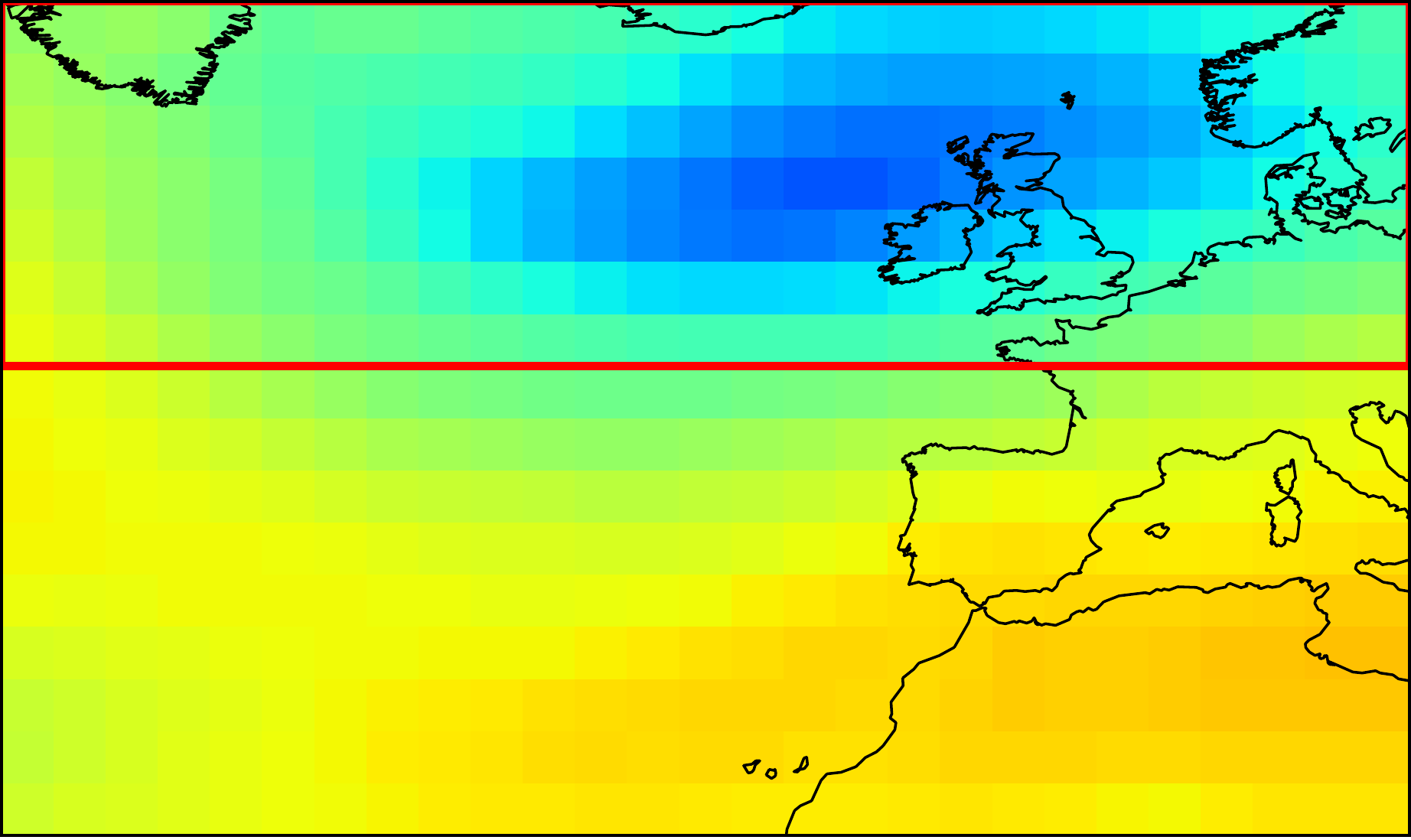} \\
        \rotatebox{90}{1989-12-22 -- 1989-12-31} & \includegraphics[width=\linewidth]{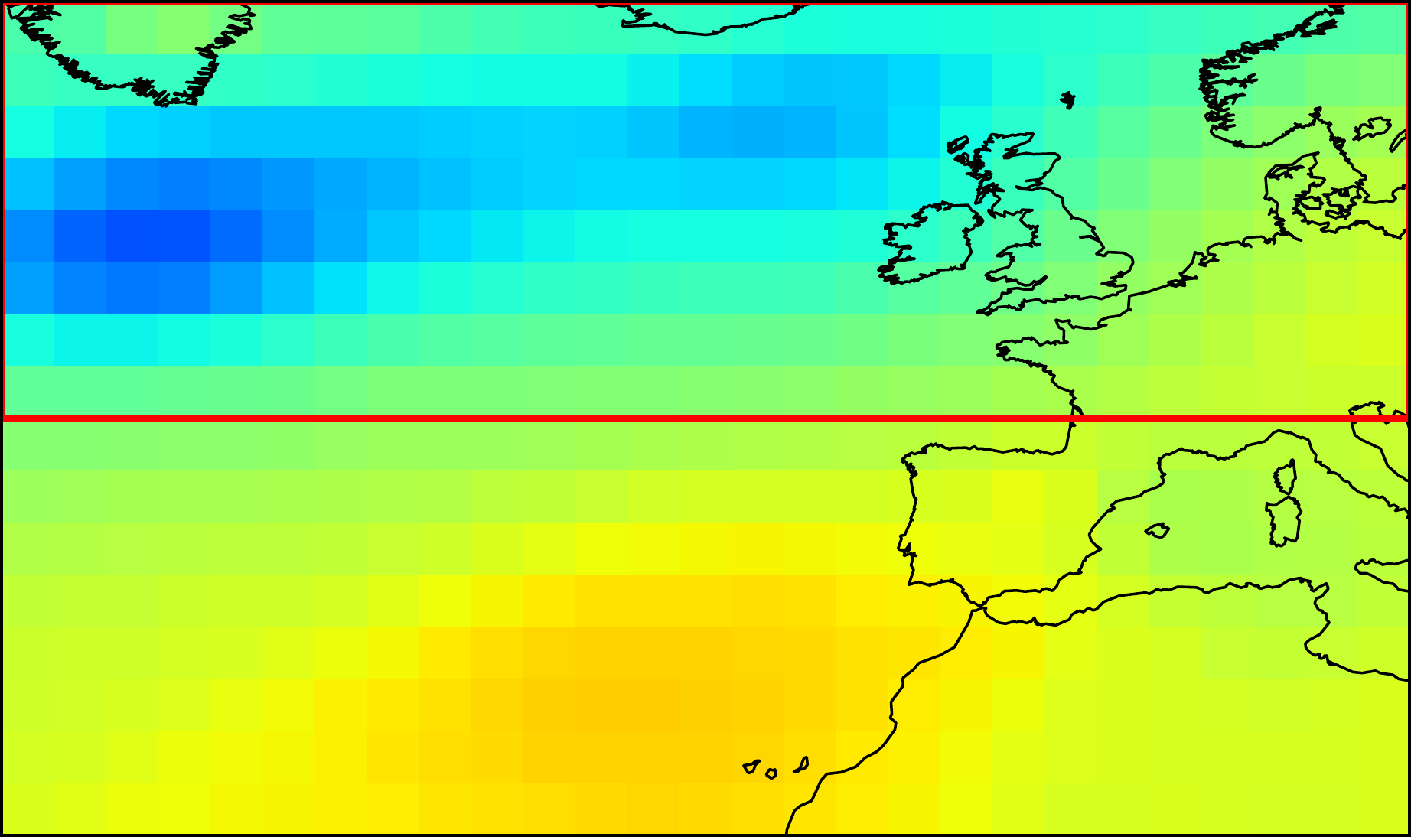} & \includegraphics[width=\linewidth]{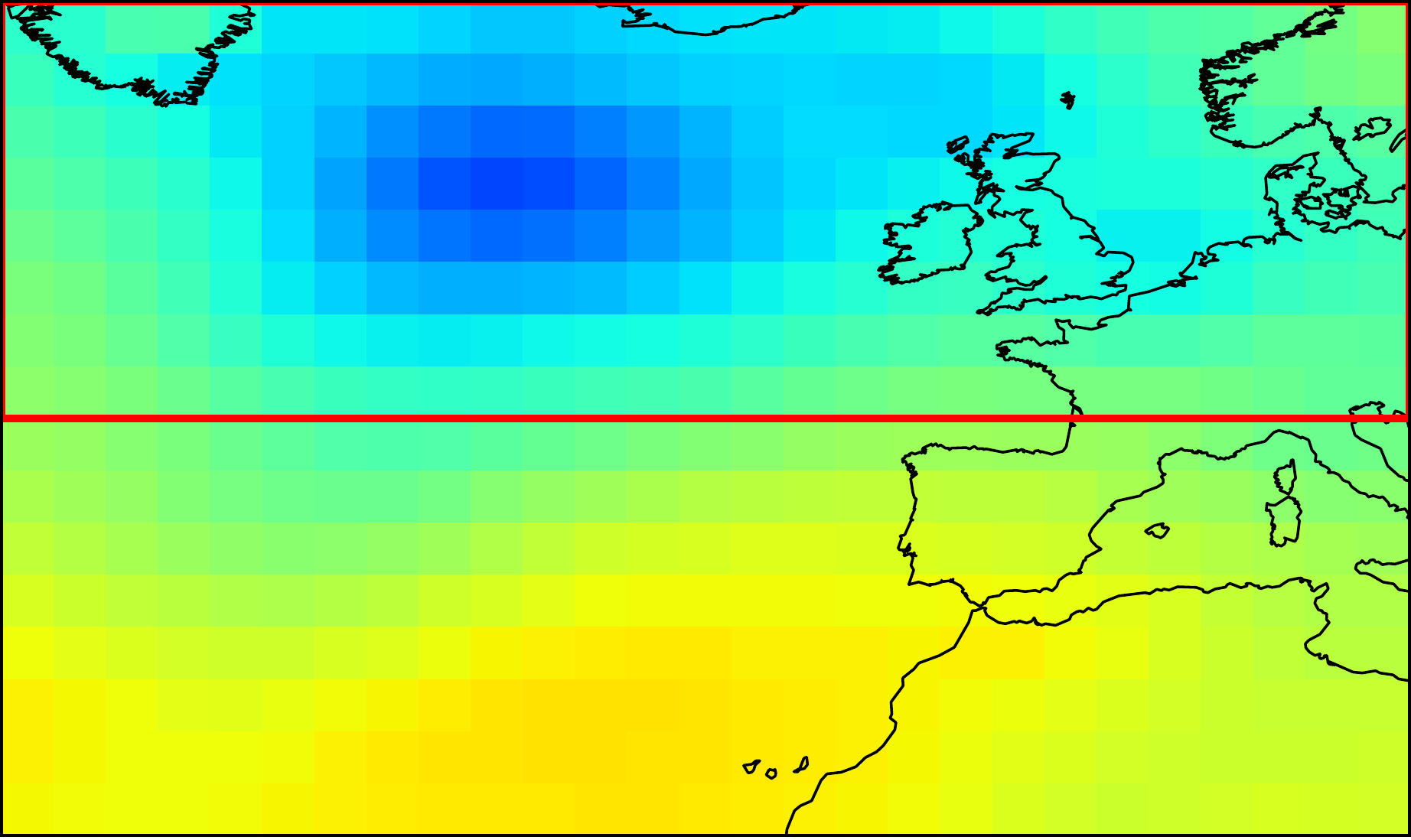} & \includegraphics[width=\linewidth]{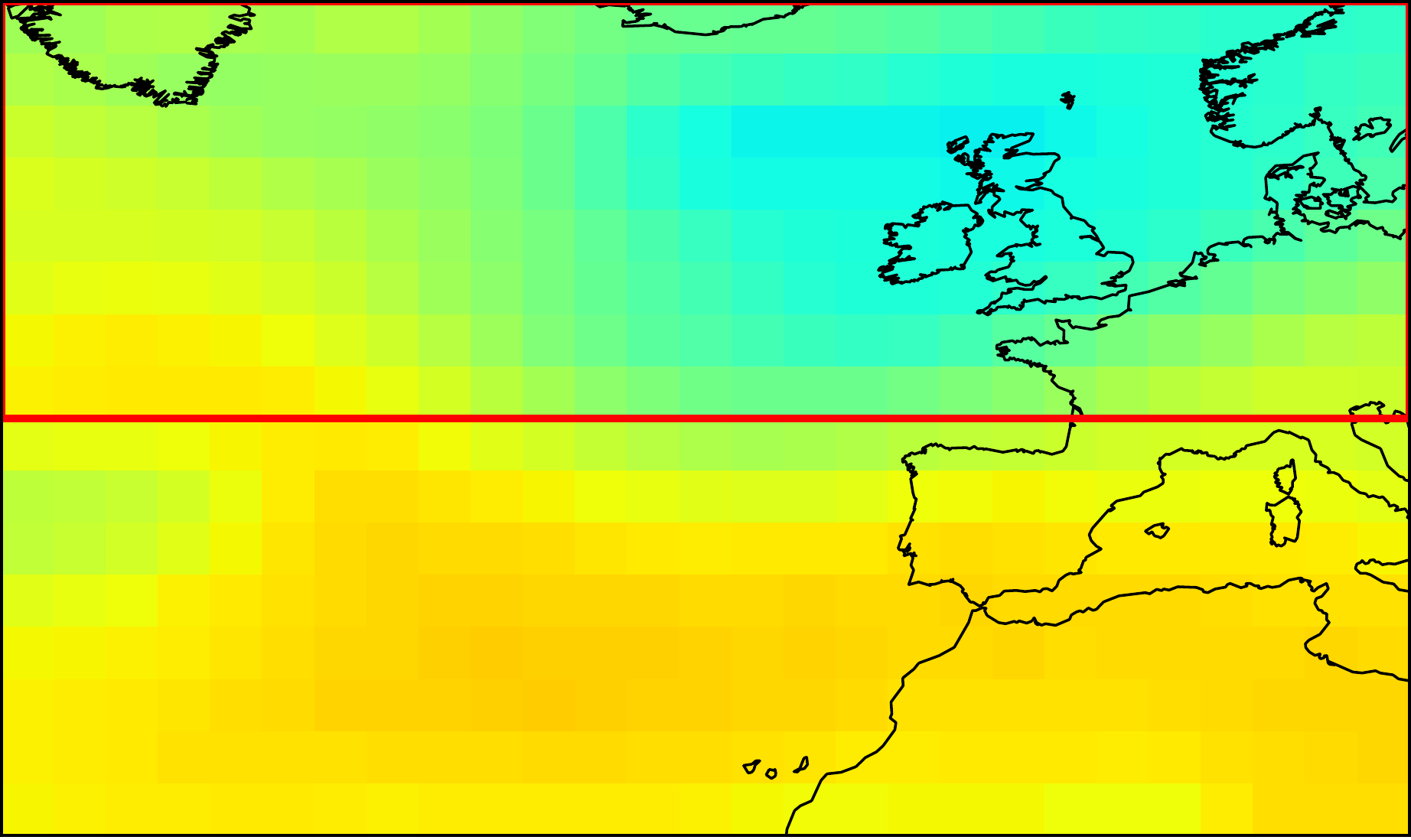} \\
        \rotatebox{90}{2009-01-18 -- 2009-01-27} & \includegraphics[width=\linewidth]{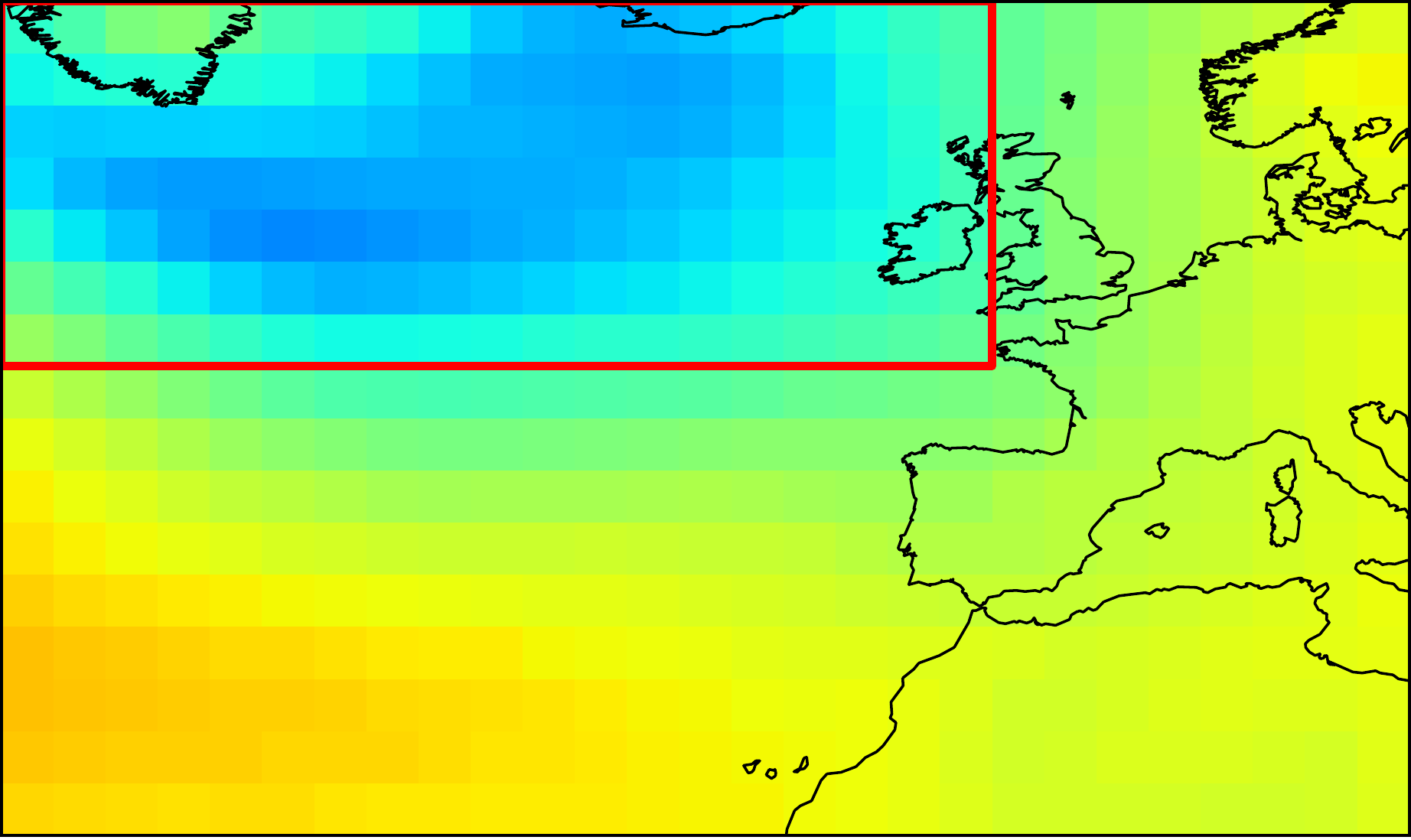} & \includegraphics[width=\linewidth]{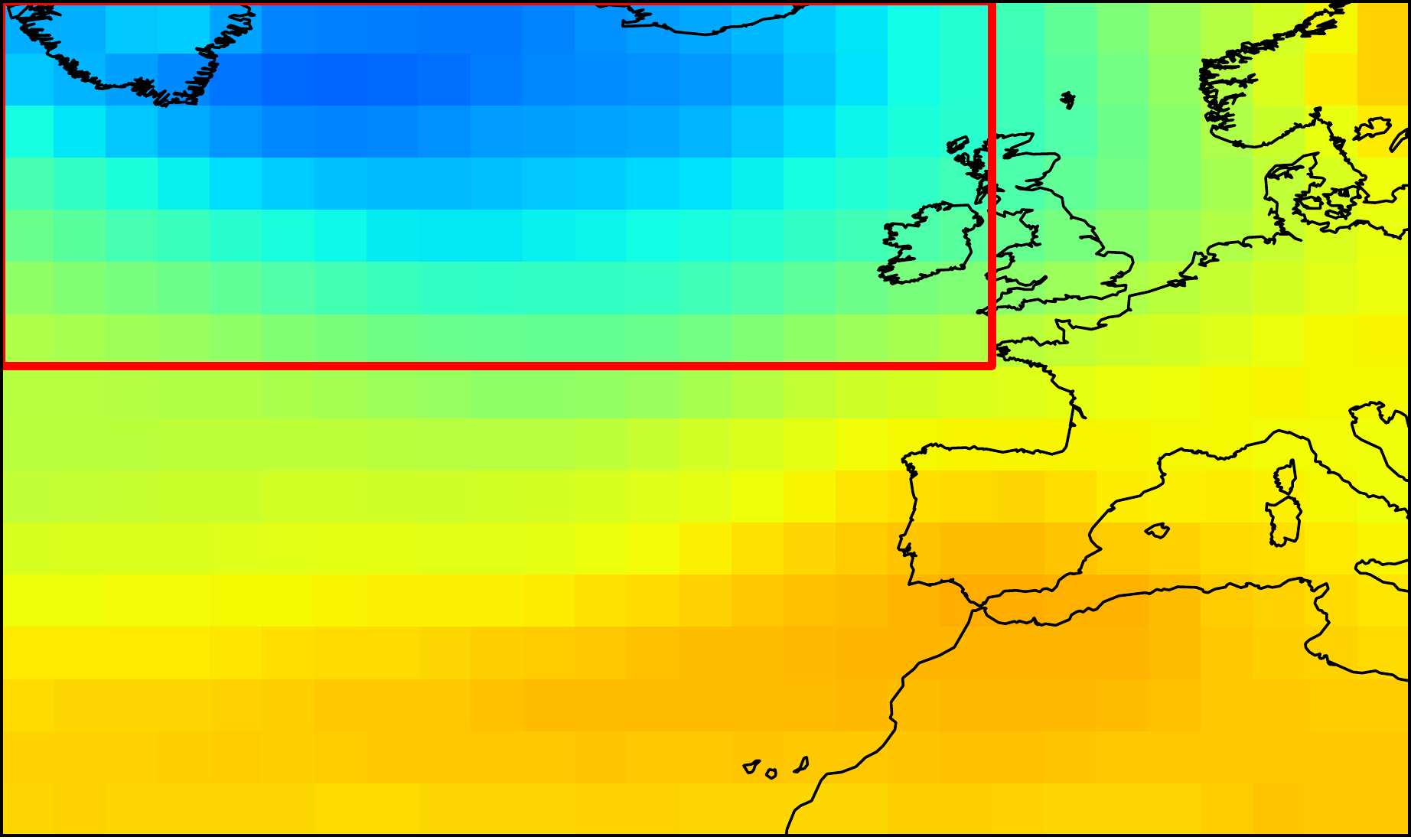} & \includegraphics[width=\linewidth]{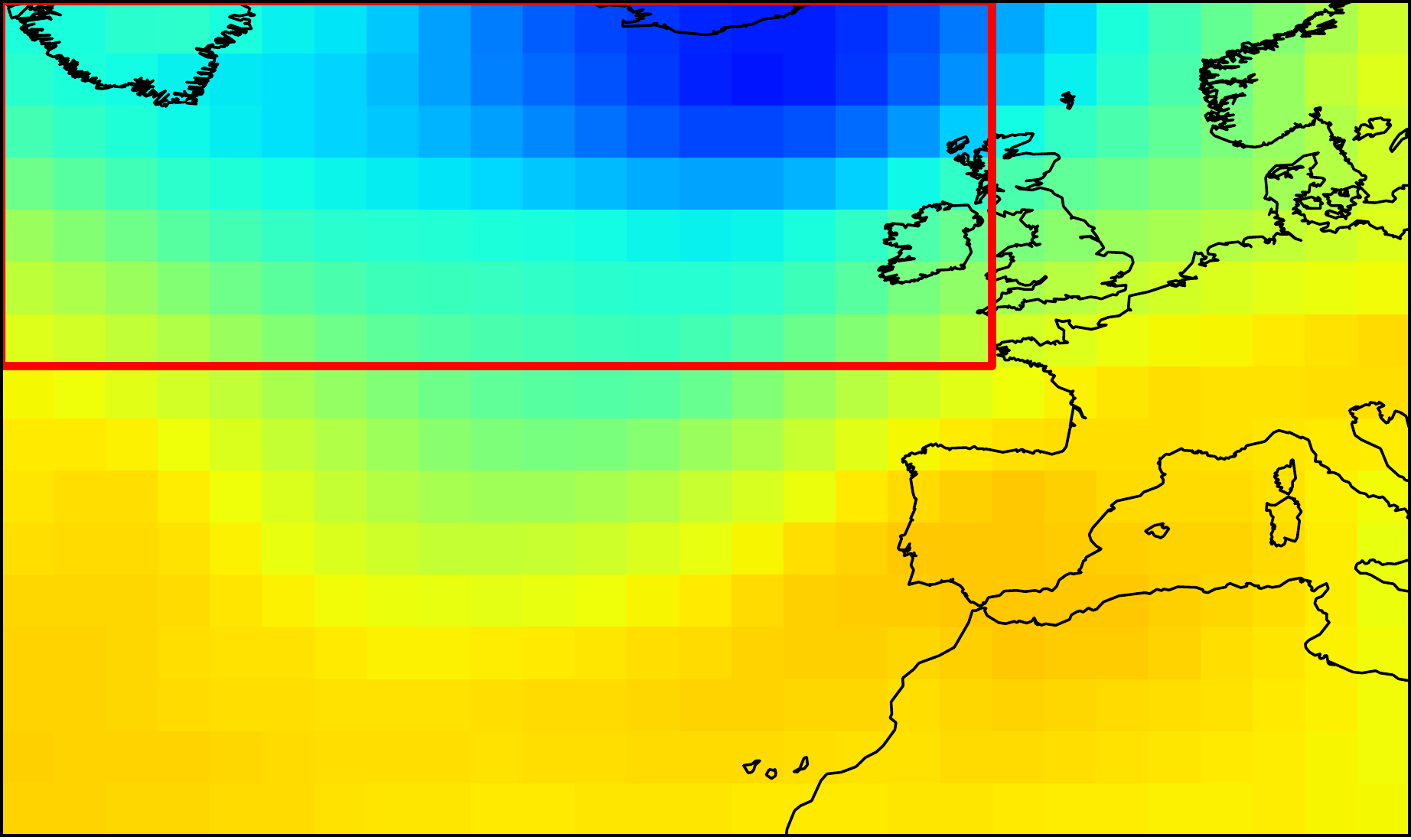} \\
    \end{tabular}\\
    \hspace{0.055\linewidth}
    \includegraphics[width=.91\linewidth]{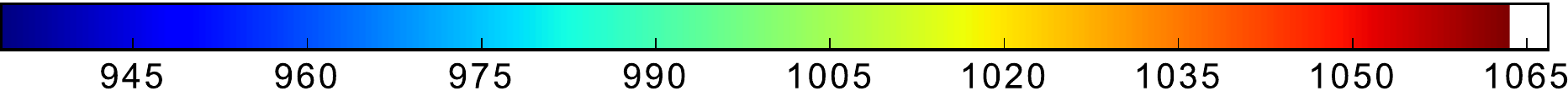}
}

\vspace{.5cm}

{
    \noindent
    \scriptsize
    \begin{tabular}{c|llll}
        \toprule 
        \# & Timeframe & Location & Score & Historical Storm \\ 
        \midrule 
        1 & 1996-01-06 -- 1996-01-15 & 40.0\textdegree\ N, -52.5\textdegree\ E -- 65.0\textdegree\ N, -2.5\textdegree\ E & 5940.691 &  \\
        2 & 1990-01-28 -- 1990-02-06 & 47.5\textdegree\ N, -52.5\textdegree\ E -- 65.0\textdegree\ N,\hphantom{-} 7.5\textdegree\ E & 5551.022 & Storm Herta (Feb 01-05) \\
        3 & 1989-12-22 -- 1989-12-31 & 45.0\textdegree\ N, -52.5\textdegree\ E -- 65.0\textdegree\ N, -2.5\textdegree\ E & 5198.513 &  \\
        4 & 2009-01-18 -- 2009-01-27 & 47.5\textdegree\ N, -52.5\textdegree\ E -- 65.0\textdegree\ N, 15.0\textdegree\ E & 4959.829 & Cyclone Joris (Jan 23) \\
        5 & 1982-12-14 -- 1982-12-23 & 50.0\textdegree\ N, -52.5\textdegree\ E -- 65.0\textdegree\ N, 15.0\textdegree\ E & 4811.575 &  \\
        6 & 1990-12-25 -- 1991-01-03 & 52.5\textdegree\ N, -52.5\textdegree\ E -- 65.0\textdegree\ N, 15.0\textdegree\ E & 4703.993 & Storm Undine (Jan 02-09) \\
        7 & 1974-01-03 -- 1974-01-12 & 47.5\textdegree\ N, -52.5\textdegree\ E -- 65.0\textdegree\ N, -5.0\textdegree\ E & 4594.737 &  \\
        8 & 1986-12-08 -- 1986-12-17 & 47.5\textdegree\ N, -52.5\textdegree\ E -- 65.0\textdegree\ N, -2.5\textdegree\ E & 4417.568 & Storm 1986/Dec (Dec 14-15) \\
        9 & 1997-12-30 -- 1998-01-08 & 50.0\textdegree\ N, -52.5\textdegree\ E -- 65.0\textdegree\ N, 10.0\textdegree\ E & 4377.532 & Cyclone Fanny (Jan 03-05) \\
        10 & 1995-01-26 -- 1995-02-04 & 47.5\textdegree\ N, -52.5\textdegree\ E -- 65.0\textdegree\ N, 15.0\textdegree\ E & 4376.735 &  \\
        11 & 2006-12-03 -- 2006-12-12 & 47.5\textdegree\ N, -52.5\textdegree\ E -- 65.0\textdegree\ N, 15.0\textdegree\ E & 4306.923 &  \\
        12 & 1997-02-18 -- 1997-02-27 & 52.5\textdegree\ N, -52.5\textdegree\ E -- 65.0\textdegree\ N, 15.0\textdegree\ E & 4249.087 &  \\
        13 & 1958-01-04 -- 1958-01-13 & 50.0\textdegree\ N, -52.5\textdegree\ E -- 65.0\textdegree\ N, 15.0\textdegree\ E & 4206.594 &  \\
        14 & 1978-12-06 -- 1978-12-15 & 45.0\textdegree\ N, -52.5\textdegree\ E -- 65.0\textdegree\ N,\hphantom{1} 2.5\textdegree\ E & 4151.843 &  \\
        15 & 1976-12-01 -- 1976-12-10 & 47.5\textdegree\ N, -52.5\textdegree\ E -- 65.0\textdegree\ N, 15.0\textdegree\ E & 4139.642 &  \\
        16 & 1971-01-18 -- 1971-01-27 & 45.0\textdegree\ N, -52.5\textdegree\ E -- 65.0\textdegree\ N, 15.0\textdegree\ E & 4030.477 &  \\
        17 & 1992-11-29 -- 1992-12-08 & 47.5\textdegree\ N, -52.5\textdegree\ E -- 65.0\textdegree\ N, 15.0\textdegree\ E & 3962.119 &  \\
        18 & 1994-01-27 -- 1994-02-05 & 50.0\textdegree\ N, -52.5\textdegree\ E -- 65.0\textdegree\ N, 15.0\textdegree\ E & 3933.832 & Cyclone Lore (Jan 27-28) \\
        19 & 2007-12-02 -- 2007-12-11 & 47.5\textdegree\ N, -52.5\textdegree\ E -- 65.0\textdegree\ N, 15.0\textdegree\ E & 3931.694 & Cyclone Fridtjof (Dec 02-03) \\
        20 & 1959-12-18 -- 1959-12-27 & 47.5\textdegree\ N, -52.5\textdegree\ E -- 65.0\textdegree\ N, 15.0\textdegree\ E & 3910.999 &  \\
        \bottomrule 
    \end{tabular} 
}

\section{Top 10 Anomalous Paragraphs in the Book Genesis}
\label{app:genesis-detections}

Each detected word sequence is shown with some \textcolor{grey}{context} colored in \textcolor{grey}{grey} before and after the detection. See \cref{sec:exp-nlp} for details on this experiment.

The text has been taken from the ``genesis'' corpus included in the \textit{Natural Language Toolkit (NLTK)} for Python and is not free of noise.

\subsection*{Detection \#1: words 3218 -- 3613 (Score: 56462.266)}

\fpar{%
    \textcolor{grey}{and called their name Adam , in the day when they were created . And Adam lived an hundred and thirty years , and begat a son in his own likeness , and after his image ; and called his name Se And the days of Adam after he had}
    begotten Seth were eight hundred yea and he begat sons and daughters : And all the days that Adam lived were nine hundred and thirty yea and he died . And Seth lived an hundred and five years , and begat Enos : And Seth lived after he begat Enos eight hundred and seven years , and begat sons and daughte And all the days of Seth were nine hundred and twelve years : and he died . And Enos lived ninety years , and begat Cainan : And Enos lived after he begat Cainan eight hundred and fifteen years , and begat sons and daughte And all the days of Enos were nine hundred and five years : and he died . And Cainan lived seventy years and begat Mahalaleel : And Cainan lived after he begat Mahalaleel eight hundred and forty years , and begat sons and daughte And all the days of Cainan were nine hundred and ten years : and he died . And Mahalaleel lived sixty and five years , and begat Jared : And Mahalaleel lived after he begat Jared eight hundred and thirty years , and begat sons and daughte And all the days of Mahalaleel were eight hundred ninety and five yea and he died . And Jared lived an hundred sixty and two years , and he begat Eno And Jared lived after he begat Enoch eight hundred years , and begat sons and daughte And all the days of Jared were nine hundred sixty and two yea and he died . And Enoch lived sixty and five years , and begat Methuselah : And Enoch walked with God after he begat Methuselah three hundred years , and begat sons and daughte And all the days of Enoch were three hundred sixty and five yea And Enoch walked with God : and he was not ; for God took him . And Methuselah lived an hundred eighty and seven years , and begat Lamech . And Methuselah lived after he begat Lamech seven hundred eighty and two years , and begat sons and daughte And all the days of Methuselah were nine hundred sixty and nine yea and he died . And Lamech lived an hundred eighty and two years , and begat a s And he called his name Noah , saying , This same
    \textcolor{grey}{shall comfort us concerning our work and toil of our hands , because of the ground which the LORD hath cursed .}
}

\subsection*{Detection \#2: words 30098 -- 30568 (Score: 41058.093)}

\fpar{%
    \textcolor{grey}{his house , and his cattle , and all his beasts , and all his substance , which he had got in the land of Canaan ; and went into the country from the face of his brother Jacob . For their riches were more than that they might dwell}
    together ; and the land wherein they were strangers could not bear them because of their cattle . Thus dwelt Esau in mount Seir : Esau is Edom . And these are the generations of Esau the father of the Edomites in mount Se These are the names of Esau ' s sons ; Eliphaz the son of Adah the wife of Esau , Reuel the son of Bashemath the wife of Esau . And the sons of Eliphaz were Teman , Omar , Zepho , and Gatam , and Kenaz . And Timna was concubine to Eliphaz Esau ' s son ; and she bare to Eliphaz Amal these were the sons of Adah Esau ' s wife . And these are the sons of Reuel ; Nahath , and Zerah , Shammah , and Mizz these were the sons of Bashemath Esau ' s wife . And these were the sons of Aholibamah , the daughter of Anah the daughter of Zibeon , Esau ' s wife and she bare to Esau Jeush , and Jaalam , and Korah . These were dukes of the sons of Esau : the sons of Eliphaz the firstborn son of Esau ; duke Teman , duke Omar , duke Zepho , duke Kenaz , Duke Korah , duke Gatam , and duke Amalek : these are the dukes that came of Eliphaz in the land of Edom ; these were the sons of Adah . And these are the sons of Reuel Esau ' s son ; duke Nahath , duke Zerah , duke Shammah , duke Mizz these are the dukes that came of Reuel in the land of Edom ; these are the sons of Bashemath Esau ' s wife . And these are the sons of Aholibamah Esau ' s wife ; duke Jeush , duke Jaalam , duke Kor these were the dukes that came of Aholibamah the daughter of Anah , Esau ' s wife . These are the sons of Esau , who is Edom , and these are their dukes . These are the sons of Seir the Horite , who inhabited the land ; Lotan , and Shobal , and Zibeon , and Anah , And Dishon , and Ezer , and Dishan : these are the dukes of the Horites , the children of Seir in the land of Edom . And the children of Lotan were Hori and Hemam ; and Lotan ' s sister was Timna . And the children of Shobal were these ; Alvan , and Manahath , and Ebal , Shepho , and Onam . And these are the children of Zibeon ; both Ajah , and Anah : this was that Anah that found the mules in the wilderness , as he fed the asses of
    \textcolor{grey}{Zibeon his father . And the children of Anah were these ; Dishon , and Aholibamah the daughter of Anah . And these are the children of Dishon ; Hemdan , and Eshban , and Ithran , and Cheran .}
}

\subsection*{Detection \#3: words 7347 -- 7684 (Score: 39679.642)}

\fpar{%
    \textcolor{grey}{may not understand one another ' s speech . So the LORD scattered them abroad from thence upon the face of all the ear and they left off to build the city . Therefore is the name of it called Babel ; because the LORD did there confound the language}
    of all the ear and from thence did the LORD scatter them abroad upon the face of all the earth . These are the generations of Shem : Shem was an hundred years old , and begat Arphaxad two years after the flo And Shem lived after he begat Arphaxad five hundred years , and begat sons and daughters . And Arphaxad lived five and thirty years , and begat Salah : And Arphaxad lived after he begat Salah four hundred and three years , and begat sons and daughters . And Salah lived thirty years , and begat Eber : And Salah lived after he begat Eber four hundred and three years , and begat sons and daughters . And Eber lived four and thirty years , and begat Peleg : And Eber lived after he begat Peleg four hundred and thirty years , and begat sons and daughters . And Peleg lived thirty years , and begat Reu : And Peleg lived after he begat Reu two hundred and nine years , and begat sons and daughters . And Reu lived two and thirty years , and begat Serug : And Reu lived after he begat Serug two hundred and seven years , and begat sons and daughters . And Serug lived thirty years , and begat Nahor : And Serug lived after he begat Nahor two hundred years , and begat sons and daughters . And Nahor lived nine and twenty years , and begat Terah : And Nahor lived after he begat Terah an hundred and nineteen years , and begat sons and daughters . And Terah lived seventy years , and begat Abram , Nahor , and Haran . Now these are the generations of Terah : Terah begat Abram , Nahor , and Haran ; and Haran begat Lot . And Haran died before his father Terah in the land of his nativity , in Ur of the Chaldees . And Abram and Nahor took them wives : the name of
    \textcolor{grey}{Abram ' s wife was Sarai ; and the name of Nahor ' s wife , Milcah , the daughter of Haran , the father of Milcah , and the father of Iscah . But Sarai was barren ; she had no child .}
}

\subsection*{Detection \#4: words 30585 -- 30993 (Score: 28796.840)}

\fpar{%
    \textcolor{grey}{And these are the children of Zibeon ; both Ajah , and Anah : this was that Anah that found the mules in the wilderness , as he fed the asses of Zibeon his father . And the children of Anah were these ; Dishon , and Aholibamah the}
    daughter of Anah . And these are the children of Dishon ; Hemdan , and Eshban , and Ithran , and Cheran . The children of Ezer are these ; Bilhan , and Zaavan , and Akan . The children of Dishan are these ; Uz , and Aran . These are the dukes that came of the Horites ; duke Lotan , duke Shobal , duke Zibeon , duke Anah , Duke Dishon , duke Ezer , duke Dishan : these are the dukes that came of Hori , among their dukes in the land of Seir . And these are the kings that reigned in the land of Edom , before there reigned any king over the children of Israel . And Bela the son of Beor reigned in Edom : and the name of his city was Dinhabah . And Bela died , and Jobab the son of Zerah of Bozrah reigned in his stead . And Jobab died , and Husham of the land of Temani reigned in his stead . And Husham died , and Hadad the son of Bedad , who smote Midian in the field of Moab , reigned in his ste and the name of his city was Avith . And Hadad died , and Samlah of Masrekah reigned in his stead . And Samlah died , and Saul of Rehoboth by the river reigned in his stead . And Saul died , and Baalhanan the son of Achbor reigned in his stead . And Baalhanan the son of Achbor died , and Hadar reigned in his ste and the name of his city was Pau ; and his wife ' s name was Mehetabel , the daughter of Matred , the daughter of Mezahab . And these are the names of the dukes that came of Esau , according to their families , after their places , by their names ; duke Timnah , duke Alvah , duke Jetheth , Duke Aholibamah , duke Elah , duke Pinon , Duke Kenaz , duke Teman , duke Mibzar , Duke Magdiel , duke Iram : these be the dukes of Edom , according to their habitations in the land of their possessi he is Esau the father of the Edomites . And Jacob dwelt in the land wherein his father was a stranger , in the land of Canaan . These are the generations of Jacob . Joseph , being
    \textcolor{grey}{seventeen years old , was feeding the flock with his brethren ; and the lad was with the sons of Bilhah , and with the sons of Zilpah , his father ' s wiv and Joseph brought unto his father their evil report .}
}

\subsection*{Detection \#5: words 40473 -- 40821 (Score: 28436.096)}

\fpar{%
    \textcolor{grey}{into Egypt , Jacob , and all his seed with h His sons , and his sons ' sons with him , his daughters , and his sons ' daughters , and all his seed brought he with him into Egypt . And these are the names of the children}
    of Israel , which came into Egypt , Jacob and his so Reuben , Jacob ' s firstborn . And the sons of Reuben ; Hanoch , and Phallu , and Hezron , and Carmi . And the sons of Simeon ; Jemuel , and Jamin , and Ohad , and Jachin , and Zohar , and Shaul the son of a Canaanitish woman . And the sons of Levi ; Gershon , Kohath , and Merari . And the sons of Judah ; Er , and Onan , and Shelah , and Pharez , and Zar but Er and Onan died in the land of Canaan . And the sons of Pharez were Hezron and Hamul . And the sons of Issachar ; Tola , and Phuvah , and Job , and Shimron . And the sons of Zebulun ; Sered , and Elon , and Jahleel . These be the sons of Leah , which she bare unto Jacob in Padanaram , with his daughter Din all the souls of his sons and his daughters were thirty and three . And the sons of Gad ; Ziphion , and Haggi , Shuni , and Ezbon , Eri , and Arodi , and Areli . And the sons of Asher ; Jimnah , and Ishuah , and Isui , and Beriah , and Serah their sist and the sons of Beriah ; Heber , and Malchiel . These are the sons of Zilpah , whom Laban gave to Leah his daughter , and these she bare unto Jacob , even sixteen souls . The sons of Rachel Jacob ' s wife ; Joseph , and Benjamin . And unto Joseph in the land of Egypt were born Manasseh and Ephraim , which Asenath the daughter of Potipherah priest of On bare unto him . And the sons of Benjamin were Belah , and Becher , and Ashbel , Gera , and Naaman , Ehi , and Rosh , Muppim , and Huppim , and Ard . These are the sons of Rachel , which were born to
    \textcolor{grey}{Jacob : all the souls were fourteen . And the sons of Dan ; Hushim . And the sons of Naphtali ; Jahzeel , and Guni , and Jezer , and Shillem . These are the sons of Bilhah , which Laban gave unto Rachel his daughter , and she}
}

\subsection*{Detection \#6: words 12299 -- 12486 (Score: 25531.722)}

\fpar{%
    \textcolor{grey}{And Abraham answered and said , Behold now , I have taken upon me to speak unto the LORD , which am but dust and ash Peradventure there shall lack five of the fifty righteous : wilt thou destroy all the city}
    for lack of five ? And he said , If I find there forty and five , I will not destroy it . And he spake unto him yet again , and said , Peradventure there shall be forty found there . And he said , I will not do it for forty ' s sake . And he said unto him , Oh let not the LORD be angry , and I will spe Peradventure there shall thirty be found there . And he said , I will not do it , if I find thirty there . And he said , Behold now , I have taken upon me to speak unto the LO Peradventure there shall be twenty found there . And he said , I will not destroy it for twenty ' s sake . And he said , Oh let not the LORD be angry , and I will speak yet but this on Peradventure ten shall be found there . And he said , I will not destroy it for ten ' s sake . And the LORD went his way
    \textcolor{grey}{, as soon as he had left communing with Abrah and Abraham returned unto his place . And there came two angels to Sodom at even ; and Lot sat in the gate of Sod and Lot seeing them rose up}
}

\subsection*{Detection \#7: words 9069 -- 9287 (Score: 25480.242)}

\fpar{%
    \textcolor{grey}{Twelve years they served Chedorlaomer , and in the thirteenth year they rebelled . And in the fourteenth year came Chedorlaomer , and the kings that were}
    with him , and smote the Rephaims in Ashteroth Karnaim , and the Zuzims in Ham , and the Emins in Shaveh Kiriathaim , And the Horites in their mount Seir , unto Elparan , which is by the wilderness . And they returned , and came to Enmishpat , which is Kadesh , and smote all the country of the Amalekites , and also the Amorites , that dwelt in Hazezontamar . And there went out the king of Sodom , and the king of Gomorrah , and the king of Admah , and the king of Zeboiim , and the king of Bela ( the same is Zoar ;) and they joined battle with them in the vale of Siddim ; With Chedorlaomer the king of Elam , and with Tidal king of nations , and Amraphel king of Shinar , and Arioch king of Ellasar ; four kings with five . And the vale of Siddim was full of slimepits ; and the kings of Sodom and Gomorrah fled , and fell there ; and they that remained fled to the mountain . And they took all the goods of Sodom and Gomorrah , and all their victuals , and went their way . And they took Lot , Abram ' s brother ' s
    \textcolor{grey}{son , who dwelt in Sodom , and his goods , and departed . And there came one that had escaped , and told Abram the Hebrew ; for he dwelt in the plain of Mamre the Amorite , brother of Eshcol , and brother of An and these were}
}

\subsection*{Detection \#8: words 6811 -- 7104 (Score: 24522.926)}

\fpar{%
    \textcolor{grey}{And the Arvadite , and the Zemarite , and the Hamathite : and afterward were the families of the Canaanites spread abroad . And the}
    border of the Canaanites was from Sidon , as thou comest to Gerar , unto Gaza ; as thou goest , unto Sodom , and Gomorrah , and Admah , and Zeboim , even unto Lasha . These are the sons of Ham , after their families , after their tongues , in their countries , and in their nations . Unto Shem also , the father of all the children of Eber , the brother of Japheth the elder , even to him were children born . The children of Shem ; Elam , and Asshur , and Arphaxad , and Lud , and Aram . And the children of Aram ; Uz , and Hul , and Gether , and Mash . And Arphaxad begat Salah ; and Salah begat Eber . And unto Eber were born two sons : the name of one was Peleg ; for in his days was the earth divided ; and his brother ' s name was Joktan . And Joktan begat Almodad , and Sheleph , and Hazarmaveth , and Jerah , And Hadoram , and Uzal , and Diklah , And Obal , and Abimael , and Sheba , And Ophir , and Havilah , and Jobab : all these were the sons of Joktan . And their dwelling was from Mesha , as thou goest unto Sephar a mount of the east . These are the sons of Shem , after their families , after their tongues , in their lands , after their nations . These are the families of the sons of Noah , after their generations , in their natio and by these were the nations divided in the earth after the flood . And the whole earth was
    \textcolor{grey}{of one language , and of one speech . And it came to pass , as they journeyed from the east , that they found a plain in the land of Shinar ; and they dwelt there .}
}

\subsection*{Detection \#9: words 11352 -- 11512 (Score: 21242.191)}

\fpar{%
    \textcolor{grey}{and I will make him a great nation . But my covenant will I establish with Isaac , which Sarah shall bear unto thee at this set time in the next year . And he left off talking with}
    him , and God went up from Abraham . And Abraham took Ishmael his son , and all that were born in his house , and all that were bought with his money , every male among the men of Abraham ' s house ; and circumcised the flesh of their foreskin in the selfsame day , as God had said unto him . And Abraham was ninety years old and nine , when he was circumcised in the flesh of his foreskin . And Ishmael his son was thirteen years old , when he was circumcised in the flesh of his foreskin . In the selfsame day was Abraham circumcised , and Ishmael his son . And all the men of his house , born in the house , and bought with money of the stranger , were circumcised with him . And the LORD appeared unto him in the plains of Mamre : and he sat in the
    \textcolor{grey}{tent door in the heat of the day ; And he lift up his eyes and looked , and , lo , three men stood by}
}

\subsection*{Detection \#10: words 22251 -- 22399 (Score: 20712.606)}

\fpar{%
    \textcolor{grey}{And give thee the blessing of Abraham , to thee , and to thy seed with thee ; that thou mayest inherit the land wherein thou art a stranger , which God}
    gave unto Abraham . And Isaac sent away Jacob : and he went to Padanaram unto Laban , son of Bethuel the Syrian , the brother of Rebekah , Jacob ' s and Esau ' s mother . When Esau saw that Isaac had blessed Jacob , and sent him away to Padanaram , to take him a wife from thence ; and that as he blessed him he gave him a charge , saying , Thou shalt not take a wife of the daughers of Canaan ; And that Jacob obeyed his father and his mother , and was gone to Padanaram ; And Esau seeing that the daughters of Canaan pleased not Isaac his father ; Then went Esau unto Ishmael , and took unto the wives which he had Mahalath the daughter of Ishmael Abraham ' s son , the sister of Nebajoth , to
    \textcolor{grey}{be his wife . And Jacob went out from Beersheba , and went toward Haran .}
}

\end{appendices}

\end{document}